\documentclass[10pt]{article}
\usepackage[utf8]{inputenc}
\usepackage[margin=1in]{geometry}
\usepackage{rotating}
\usepackage{booktabs}
\usepackage{mathtools}
\usepackage{array}
\usepackage{amsmath, amsfonts, amssymb}
\usepackage{color}
\usepackage{hyperref}
\usepackage{bbm}
\usepackage{algorithm}
\usepackage[noend]{algpseudocode}
\usepackage{subcaption}
\usepackage{graphbox}
\usepackage[dvipsnames]{xcolor}
\usepackage{tikz}
\usepackage[frozencache,cachedir=.]{minted}
\usepackage{mdframed}
\definecolor{light-gray}{HTML}{F7F7F7}
\definecolor{frame-color}{HTML}{CFCFCF}
\surroundwithmdframed[
backgroundcolor=light-gray,%
linecolor=frame-color%
]{minted}

\usepackage{mathtools}
\mathtoolsset{showonlyrefs}

\usepackage[
    backend=biber,
    style=ieee,
    citestyle=numeric-comp,
    sorting=none,
    giveninits=true,
    natbib,
    hyperref,
    maxbibnames=99,
    doi=false,isbn=false,url=false,eprint=false
]{biblatex}

\usepackage{amsthm}
\usepackage[normalem]{ulem}
\usepackage{soul}

\newtheorem{theorem}{Theorem} 
\newtheorem{prop}{Proposition}

\newtheorem{definition}{Definition}

\usepackage{chngcntr}
\usepackage{apptools}
\AtAppendix{\counterwithin{theorem}{section}}
\AtAppendix{\counterwithin{prop}{section}}

\usepackage{mathtools}

\def\Ical{\mathcal{I}}

\def\R{\mathbb{R}}
\def\E{\mathbb{E}}

\newcommand{\Y}{\mathcal{Y}}
\newcommand{\X}{\mathcal{X}}
\newcommand{\T}{\mathcal{T}}
\newcommand{\C}{\mathcal{C}}

\newcommand{\lhat}{\hat{\lambda}}
\newcommand{\Lhat}{\widehat{\Lambda}}
\newcommand{\Rhat}{\widehat{R}}
\newcommand{\ind}[1]{\mathbbm{1}\left\{#1\right\}}

\newcommand{\Hlam}{\mathcal{H}_ {\lambda}}

\renewcommand{\P}{\mathbb{P}}
\newcommand{\Clam}{\mathcal{C}_{\lambda}}
\newcommand{\Clamhat}{\mathcal{C}_{\hat{\lambda}}}
\newcommand{\Tlam}{\mathcal{T}_{\lambda}}

\newcommand{\quantile}{\mathrm{Quantile}}

\newcommand{\jupyter}[1]{\href{#1}{\begingroup
\setbox0=\hbox{\includegraphics[height=1.5em]{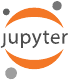}}%
\parbox{\wd0}{\box0}\endgroup}}

\makeatletter
\def\blfootnote{\xdef\@thefnmark{}\@footnotetext}
\makeatother

\title{A Gentle Introduction to Conformal Prediction and Distribution-Free Uncertainty Quantification}
\author{Anastasios N. Angelopoulos and Stephen Bates}
\date{\vspace{-0.3cm} {\small \today} \vspace{-0.2cm}}

\begin{document}

\maketitle

\begin{abstract}
Black-box machine learning models are now routinely used in high-risk settings, like medical diagnostics, which demand uncertainty quantification to avoid consequential model failures.
Conformal prediction (a.k.a. conformal inference) is a user-friendly paradigm for creating statistically rigorous uncertainty sets/intervals for the predictions of such models.
Critically, the sets are valid in a \emph{distribution-free} sense: they possess explicit, non-asymptotic guarantees even without distributional assumptions or model assumptions.
One can use conformal prediction with any pre-trained model, such as a neural network, to produce sets that are guaranteed to contain the ground truth with a user-specified probability, such as $90\%$.
It is easy-to-understand, easy-to-use, and general, applying naturally to problems arising in the fields of computer vision, natural language processing, deep reinforcement learning, and so on.

This hands-on introduction is aimed to provide the reader a working understanding of conformal prediction and related distribution-free uncertainty quantification techniques with one self-contained document.
We lead the reader through practical theory for and examples of conformal prediction and describe its extensions to complex machine learning tasks involving structured outputs, distribution shift, time-series, outliers, models that abstain, and more.
Throughout, there are many explanatory illustrations, examples, and code samples in Python.
With each code sample comes a Jupyter notebook implementing the method on a real-data example; the notebooks can be accessed and easily run by clicking on the following icons: \jupyter{https://github.com/aangelopoulos/conformal-prediction}.
\end{abstract}

\newpage
\tableofcontents

\newpage
\section{Conformal Prediction}
\label{sec:conformal}
\vspace{-0.4cm}

\begin{figure}[t]
\centering
    \includegraphics[width=5in]{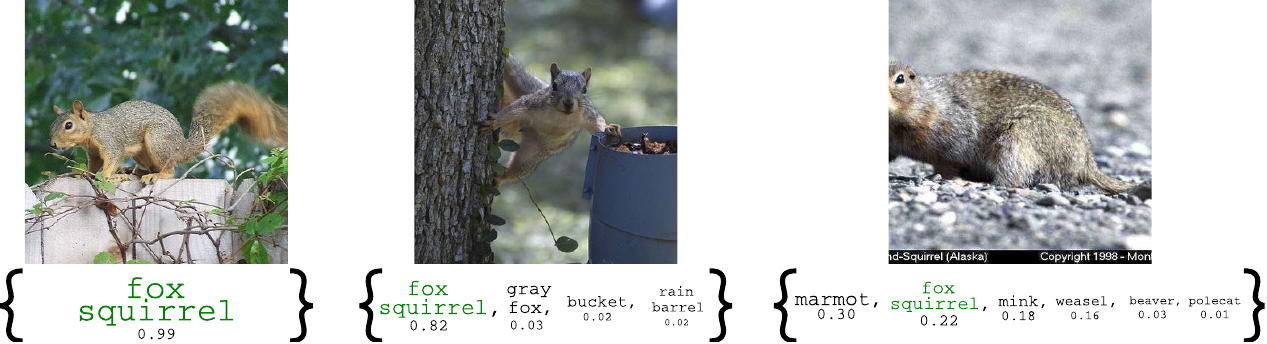}
    \caption{\textbf{Prediction set examples on Imagenet.} We show three progressively more difficult examples of the class \texttt{fox squirrel} and the prediction sets (i.e., $\C(X_{\rm test})$) generated by conformal prediction.}
    \label{fig:setfigure}
\end{figure}

Conformal prediction~\cite{vovk2005algorithmic, papadopoulos2002inductive, lei2014distribution} (a.k.a. conformal inference) is a straightforward way to generate prediction sets for any model.
We will introduce it with a short, pragmatic image classification example, and follow up in later paragraphs with a general explanation. The high-level outline of conformal prediction is as follows. First, we begin with a fitted predicted model (such as a neural network classifier) which we will call $\hat{f}$. Then, we will create prediction sets (a set of possible labels) for this classifier using a small amount of additional \emph{calibration data}---we will sometimes call this the \emph{calibration step}. 

Formally, suppose we have images as input and they each contain one of $K$ classes. 
We begin with a classifier that outputs estimated probabilities (softmax scores) for each class: $\hat{f}(x) \in  [0,1]^K$. 
Then, we reserve a moderate number (e.g., 500) of fresh i.i.d. pairs of images and classes unseen during training, $(X_1,Y_1),\dots,(X_n,Y_n)$, for use as calibration data. 
Using $\hat{f}$ and the calibration data, we seek to construct a \emph{prediction set} of possible labels $\C(X_{\rm test}) \subset \{1,\dots,K\}$ that is valid in the following sense:
\begin{equation}
    \label{eq:coverage}
    1-\alpha \leq \P(Y_{\rm test} \in \C(X_{\rm test})) \leq 1-\alpha+\frac{1}{n+1},
\end{equation}
where $(X_{\rm test}, Y_{\rm test})$ is a fresh test point from the same distribution, and $\alpha \in [0,1]$ is a user-chosen error rate. 
In words, the probability that the prediction set contains the correct label is almost exactly $1-\alpha$; we call this property \emph{marginal coverage}, since the probability is marginal (averaged) over the randomness in the calibration and test points.
See Figure~\ref{fig:setfigure} for examples of prediction sets on the Imagenet dataset.

\begin{figure}[b]
    \begin{minipage}{0.4\textwidth}
        \includegraphics[align=c,width=\textwidth]{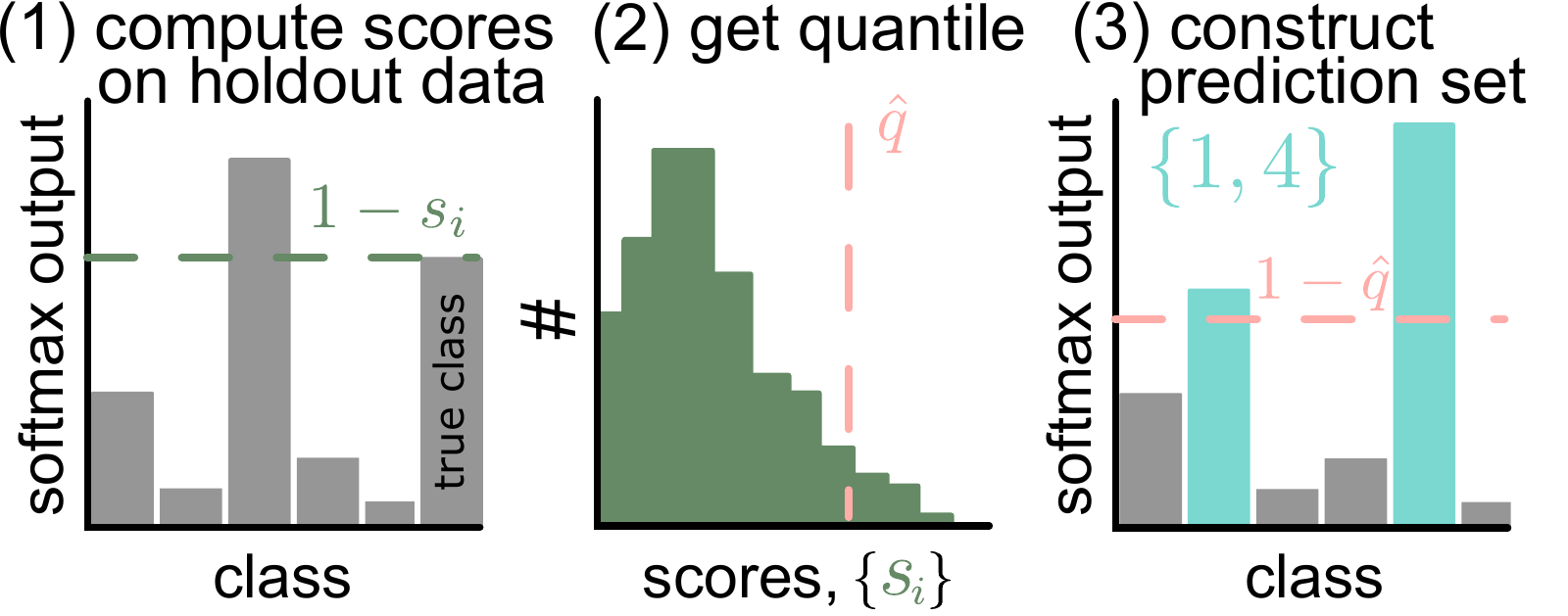}
    \end{minipage}
    \begin{minipage}{0.6\textwidth}
        \begin{minted}[fontsize=\footnotesize]{python}
# 1: get conformal scores. n = calib_Y.shape[0]
cal_smx = model(calib_X).softmax(dim=1).numpy()
cal_scores = 1-cal_smx[np.arange(n),cal_labels]
# 2: get adjusted quantile
q_level = np.ceil((n+1)*(1-alpha))/n
qhat = np.quantile(cal_scores, q_level, method='higher')
val_smx = model(val_X).softmax(dim=1).numpy()
prediction_sets = val_smx >= (1-qhat) # 3: form prediction sets
        \end{minted}
    \end{minipage}
    \caption{\textbf{Illustration of conformal prediction} with matching Python code. \jupyter{https://github.com/aangelopoulos/conformal-prediction/blob/main/notebooks/imagenet-smallest-sets.ipynb}}
    \label{fig:teaser}
\end{figure}

To construct $\C$ from $\hat{f}$ and the calibration data, we will perform a simple calibration step that requires only a few lines of code; see the right panel of Figure~\ref{fig:teaser}. 
We now describe the calibration step in more detail, introducing some terms that will be helpful later on. 
First, we set the \emph{conformal score} $s_i=1-\hat{f}(X_i)_{Y_i}$ to be one minus the softmax output of the true class. 
The score is high when the softmax output of the true class is low, i.e., when the model is badly wrong.
Next comes the critical step: define $\hat{q}$ to be the $\lceil(n+1)(1-\alpha)\rceil/n$ empirical quantile of $s_1,...,s_n$, where $\lceil \cdot \rceil$ is the ceiling function
($\hat{q}$ is essentially the $1-\alpha$ quantile, but with a small correction). 
Finally, for a new test data point (where $X_{\rm test}$ is known but $Y_{\rm test}$ is not), create a prediction set $\C(X_{\rm test})=\{y : \hat{f}(X_{\rm test})_y \geq 1-\hat{q}\}$ that includes all classes with a high enough softmax output (see Figure~\ref{fig:teaser}).
Remarkably, this algorithm gives prediction sets that are guaranteed to satisfy~\eqref{eq:coverage}, no matter what (possibly incorrect) model is used or what the (unknown) distribution of the data is.

\subsubsection*{Remarks}

Let us think about the interpretation of $\C$.
The function $\C$ is \emph{set-valued}---it takes in an image, and it outputs a set of classes as in Figure~\ref{fig:setfigure}.
The model's softmax outputs help to generate the set.
This method constructs a different output set \emph{adaptively to each particular input}.
The sets become larger when the model is uncertain or the image is intrinsically hard.
This is a property we want, because the size of the set gives you an indicator of the model's certainty.
Furthermore, $\C(X_{\rm test})$ can be interpreted as a set of plausible classes that the image $X_{\rm test}$ could be assigned to.
Finally, $\C$ is \emph{valid}, meaning it satisfies~\eqref{eq:coverage}.\footnote{Due to the discreteness of $Y$, a small modification involving tie-breaking is needed to additionally satisfy the upper bound (see~\cite{angelopoulos2020sets} for details; this randomization is usually ignored in practice). We will henceforth ignore such tie-breaking.}
These properties of $\C$ translate naturally to other machine learning problems, like regression, as we will see.

With an eye towards generalization, let us review in detail what happened in our classification problem. To begin, we were handed a model that had an inbuilt, but heuristic, notion of uncertainty: softmax outputs.
The softmax outputs attempted to measure the conditional probability of each class; in other words, the $j$th entry of the softmax vector estimated $\P(Y=j \mid X=x)$, the probability of class $j$ conditionally on an input image $x$.
However, we had no guarantee that the softmax outputs were any good; they may have been arbitrarily overfit or otherwise untrustworthy. Therefore, instead of taking the softmax outputs at face value, we used the holdout set to adjust for their deficiencies. 

The holdout set contained $n\approx 500$ fresh data points that the model never saw during training, which allowed us to get an honest appraisal of its performance.
The adjustment involved computing conformal scores, which grow when the model is uncertain, but are not valid prediction intervals on their own.
In our case, the conformal score was one minus the softmax output of the true class, but in general, the score can be any function of $x$ and $y$.
We then took $\hat{q}$ to be roughly the $1-\alpha$ quantile of the scores.
In this case, the quantile had a simple interpretation---when setting $\alpha=0.1$, at least $90\%$ of ground truth softmax outputs are guaranteed to be above the level $1-\hat{q}$ (we prove this rigorously in Appendix~\ref{app:coverage-proof}).
Taking advantage of this fact, at test-time, we got the softmax outputs of a new image $X_{\rm test}$ and collected all classes with outputs above $1-\hat{q}$ into a prediction set $\C(X_{\rm test})$.
Since the softmax output of the new true class $Y_{\rm test}$ is guaranteed to be above $1-\hat{q}$ with probability at least $90\%$, we finally got the guarantee in Eq.~\eqref{eq:coverage}.

\subsection{Instructions for Conformal Prediction}
\label{subsec:split_conformal_alg}

As we said during the summary, conformal prediction is not specific to softmax outputs or classification problems.
In fact, conformal prediction can be seen as a method for taking \textbf{any heuristic notion of uncertainty} from \textbf{any model} and converting it to a rigorous one (see the diagram below).
Conformal prediction does not care if the underlying prediction problem is discrete/continuous or classification/regression.
\begin{figure}[h!]
    \centering
    \includegraphics[width=0.3\textwidth]{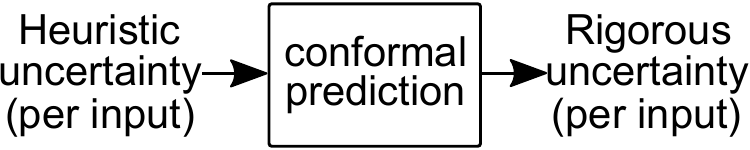}
\end{figure}

We next outline conformal prediction for a general input $x$ and output $y$ (not necessarily discrete).
\begin{enumerate}
    \item Identify a heuristic notion of uncertainty using the pre-trained model.
    \item Define the score function $s(x,y) \in \R$. (Larger scores encode worse agreement between $x$ and $y$.)
    \item Compute $\hat{q}$ as the $\frac{\lceil (n+1)(1-\alpha) \rceil}{n}$ quantile of the calibration scores $s_1=s(X_1,Y_1),...,s_n=s(X_n,Y_n)$.
    \item Use this quantile to form the prediction sets for new examples:
    \begin{equation} \label{eq:conformal_set} \C(X_{\rm test}) = \left\{y : s(X_{\rm test},y) \le \hat{q}\right\}.\end{equation}
\end{enumerate}
\noindent As before, these sets satisfy the validity property in~\eqref{eq:coverage}, for any (possibly uninformative) score function and (possibly unknown) distribution of the data. We formally state the coverage guarantee next. 
\begin{theorem}[Conformal coverage guarantee; Vovk, Gammerman, and Saunders~\citep{vovk1999machine}]
\label{thm:conformal_calibration}
Suppose $(X_i, Y_i)_{i = 1,\dots,n}$ and $(X_{\rm test}, Y_{\rm test})$ are i.i.d. and define $\hat{q}$ as in step 3 above and $\C(X_{\rm test})$ as in step 4 above. Then the following holds:
\begin{equation*}
P\Big(Y_{\rm test} \in \C(X_{\rm test})\Big) \ge 1 - \alpha.
\end{equation*}
\end{theorem}
See Appendix~\ref{app:coverage-proof} for a proof and a statement that includes the upper bound in~\eqref{eq:coverage}. 
We note that the above is only a special case of conformal prediction, called \emph{split conformal prediction}. This is the most widely-used version of conformal prediction, and it will be our primary focus. 
To complete the picture, we describe conformal prediction in full generality later in Section~\ref{sec:full_conformal} and give an overview of the literature in Section~\ref{sec:history}.

\subsubsection*{Choice of score function}
Upon first glance, this seems too good to be true, and a skeptical reader might ask the following question:
\begin{center}
    \emph{How is it possible to construct a statistically valid prediction set even if the heuristic notion of uncertainty of the underlying model is arbitrarily bad?}
\end{center}

Let's give some intuition to supplement the mathematical understanding from the proof in Appendix~\ref{app:coverage-proof}.
Roughly, if the scores $s_i$ correctly rank the inputs from lowest to highest magnitude of model error, then the resulting sets will be smaller for easy inputs and bigger for hard ones.
If the scores are bad, in the sense that they do not approximate this ranking, then the sets will be useless.
For example, if the scores are random noise, then the sets will contain a random sample of the label space, where that random sample is large enough to provide valid marginal coverage.
This illustrates an important underlying fact about conformal prediction: although the guarantee always holds, \textbf{the usefulness of the prediction sets is primarily determined by the score function}.
This should be no surprise---the score function incorporates almost all the information we know about our problem and data, including the underlying model itself.
For example, the main difference between applying conformal prediction on classification problems versus regression problems is the choice of score.
There are also many possible score functions for a single underlying model, which have different properties. Therefore, constructing the right score function is an important engineering choice.
We will next show a few examples of good score functions.

\section{Examples of Conformal Procedures}
\label{sec:examples-conformal}
In this section we give examples of conformal prediction applied in many settings, with the goal of providing the reader a bank of techniques to practically deploy.
Note that we will focus only on one-dimensional $Y$ in this section, and smaller conformal scores will correspond to more model confidence (such scores are called \textit{nonconformity} scores). 
Richer settings, such as high-dimensional $Y$, complicated (or multiple) notions of error, or where different mistakes cost different amounts, often require the language of \emph{risk control}, outlined in Section~\ref{app:ltt}.

\subsection{Classification with Adaptive Prediction Sets}
\label{subsec:aps}

\begin{figure}[t]
    \centering
    \begin{minted}[fontsize=\footnotesize]{python}
# Get scores. calib_X.shape[0] == calib_Y.shape[0] == n
cal_pi = cal_smx.argsort(1)[:,::-1]; cal_srt = np.take_along_axis(cal_smx,cal_pi,axis=1).cumsum(axis=1)
cal_scores = np.take_along_axis(cal_srt,cal_pi.argsort(axis=1),axis=1)[range(n),cal_labels]
# Get the score quantile
qhat = np.quantile(cal_scores, np.ceil((n+1)*(1-alpha))/n, interpolation='higher')
# Deploy (output=list of length n, each element is tensor of classes)
val_pi = val_smx.argsort(1)[:,::-1]; val_srt = np.take_along_axis(val_smx,val_pi,axis=1).cumsum(axis=1)
prediction_sets = np.take_along_axis(val_srt <= qhat,val_pi.argsort(axis=1),axis=1)
    \end{minted}
    \caption{\textbf{Python code for adaptive prediction sets.} \jupyter{https://github.com/aangelopoulos/conformal-prediction/blob/main/notebooks/imagenet-aps.ipynb}}
    \label{fig:aps-code}
\end{figure}

Let's begin our sequence of examples with an improvement to the classification example in Section~\ref{sec:conformal}.
The previous method produces prediction sets with the smallest average size~\cite{Sadinle2016LeastAS}, but it tends to undercover hard subgroups and overcover easy ones.
Here we develop a different method called \emph{adaptive prediction sets} (APS) that avoids this problem.
We will follow~\cite{romano2020classification} and~\cite{angelopoulos2020sets}.

As motivation for this new procedure, note that if the softmax outputs $\hat{f}(X_{\rm test})$ were a perfect model of $Y_{\rm test}|X_{\rm test}$, we would greedily include the top-scoring classes until their total mass just exceeded $1-\alpha$.
Formally, we can describe this oracle algorithm as
\begin{equation}
    \left\{ \pi_1(x), ..., \pi_k(x) \right\}\text{, where }k = \sup \left\{ k' : \sum\limits_{j=1}^{k'} \hat{f}(X_{\rm test})_{\pi_j(x)} < 1-\alpha \right\} + 1, 
    \label{eq:oracle-romano}
\end{equation}
and $\pi(x)$ is the permutation of $\{1, ..., K\}$ that sorts $\hat{f}(X_{\rm test})$ from most likely to least likely.
In practice, however, this procedure fails to provide coverage, since $\hat{f}(X_{\rm test})$ is not perfect; it only provides us a heuristic notion of uncertainty. Therefore, we will use conformal prediction to turn this into a rigorous notion of uncertainty.

To proceed, we define a score function inspired by the oracle algorithm:
\begin{equation}
    s(x,y) = \sum\limits_{j=1}^k \hat{f}(x)_{\pi_j(x)}\text{, where }y=\pi_k(x).
\end{equation}
In other words, we greedily include classes in our set until we reach the true label, then we stop.
Unlike the score from Section~\ref{sec:conformal}, this one utilizes the softmax outputs of all classes, not just the true class.

The next step, as in all conformal procedures, is to set $\hat{q}=\quantile(s_1,...,s_n\; ; \; \frac{\lceil (n+1)(1-\alpha) \rceil}{n})$.
Having done so, we will form the prediction set $\{y : s(x,y) \leq \hat{q} \}$, modified slightly to avoid zero-size sets:
\begin{equation}
    \C(x) = \left\{\pi_1(x), ..., \pi_k(x)\right\}\text{, where }k = \sup \left\{ k' : \sum\limits_{j=1}^{k'} \hat{f}(x)_{\pi_j(x)} < \hat{q} \right\} + 1
    .
    \label{eq:aps-sets}
\end{equation}
Figure~\ref{fig:aps-code} shows Python code to implement this method. 
As usual, these uncertainty sets (with tie-breaking) satisfy~\eqref{eq:coverage}.
See~\cite{angelopoulos2020sets} for details and significant practical improvements, which we implemented here: \jupyter{https://github.com/aangelopoulos/conformal-prediction/blob/main/notebooks/imagenet-raps.ipynb}.
\begin{figure}[h]
    \centering
    \includegraphics[width=0.6\textwidth]{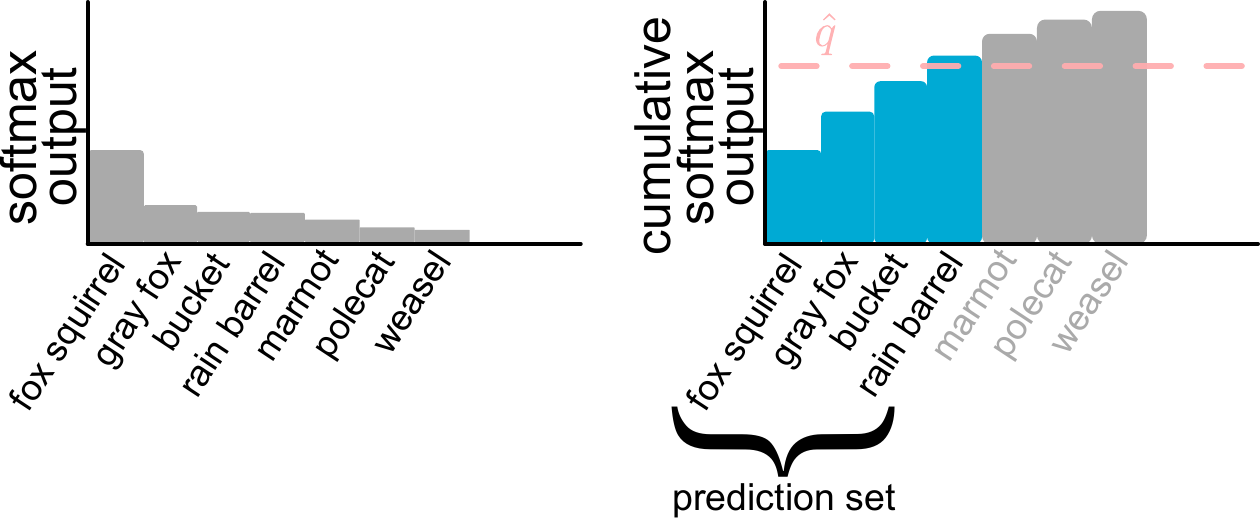}
    \caption{\textbf{A visualization of the adaptive prediction sets algorithm} in Eq.~\eqref{eq:aps-sets}.
    Classes are included from most to least likely until their cumulative softmax output exceeds the quantile.}
    \label{fig:aps-diagram}
\end{figure}
\vspace{-0.5cm}
\subsection{Conformalized Quantile Regression}
\label{subsec:cqr}
We will next show how to incorporate uncertainty into regression problems with a continuous output, following the algorithm in~\cite{romano2019conformalized}. 
We use quantile regression~\cite{koenker1978regression} as our base model.
As a reminder, the quantile regression algorithm attempts to learn the $\gamma$ quantile of $Y_{\rm test}|X_{\rm test}=x$ for each possible value of $x$. 
We will call the true quantile $t_\gamma(x)$ and the fitted model $\hat{t}_\gamma(x)$.
Since by definition $Y_{\rm test}|X_{\rm test}=x$ lands below $t_{0.05}(x)$ with $5\%$ probability and above $t_{0.95}(x)$ with $5\%$ probability, we would expect the interval $\left[\hat{t}_{0.05}(x),\hat{t}_{0.95}(x)\right]$ to have approximately 90\% coverage.
However, because the fitted quantiles may be inaccurate, we will conformalize them.
Python pseudocode for conformalized quantile regression is in Figure~\ref{fig:cqr-code}.

After training an algorithm to output two such quantiles (this can be done with a standard loss function, see below), $t_{\alpha/2}$ and $t_{1-\alpha/2}$, we can define the score to be the difference between $y$ and its nearest quantile,
\begin{equation}
    s(x,y) = \max\left\{ \hat{t}_{\alpha/2}(x)-y, y-\hat{t}_{1-\alpha/2}(x) \right\}.
\end{equation}
After computing the scores on our calibration set and setting $\hat{q}=\quantile(s_1,...,s_n\; ; \; \frac{\lceil (n+1)(1-\alpha) \rceil}{n})$, we can form valid prediction intervals by taking
\begin{equation}
    \C(x) = \left[\hat{t}_{\alpha/2}(x)-\hat{q},\hat{t}_{1-\alpha/2}(x)+\hat{q}\right].
    \label{eq:cqr-sets}
\end{equation}
Intuitively, the set $\C(x)$ just grows or shrinks the distance between the quantiles by $\hat{q}$ to achieve coverage.

\begin{figure}
    \centering
    \begin{minted}[fontsize=\footnotesize]{python}
# Get scores
cal_scores = np.maximum(cal_labels-model_upper(cal_X), model_lower(cal_X)-cal_labels)
# Get the score quantile
qhat = np.quantile(cal_scores, np.ceil((n+1)*(1-alpha))/n, interpolation='higher')
# Deploy (output=lower and upper adjusted quantiles)
prediction_sets = [val_lower - qhat, val_upper + qhat]
    \end{minted}
    \vspace{-0.5cm}
    \caption{\textbf{Python code for conformalized quantile regression.} \jupyter{https://github.com/aangelopoulos/conformal-prediction/blob/main/notebooks/meps-cqr.ipynb}}
    \label{fig:cqr-code}
    \vspace{-0.6cm}
\end{figure}

\begin{figure}[h]
    \centering
    \includegraphics[width=0.28\textwidth]{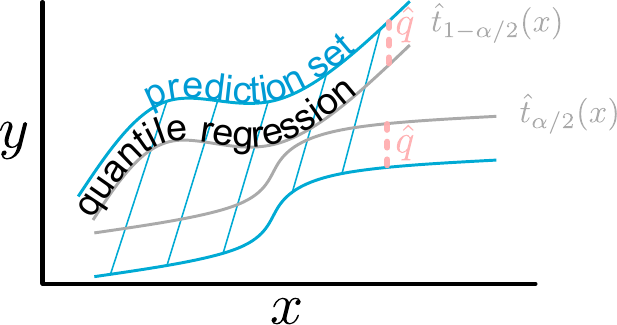}
    \caption{\textbf{A visualization of the conformalized quantile regrssion algorithm} in Eq.~\eqref{eq:cqr-sets}.
    We adjust the quantiles by the constant $\hat{q}$, picked during the calibration step.}
    \label{fig:cqr-diagram}
\end{figure}

As before, $\C$ satisfies the coverage property in Eq.~\eqref{eq:coverage}.
However, unlike our previous example in Section~\ref{sec:conformal}, $\C$ is no longer a set of classes, but instead a \emph{continuous interval} in $\R$.
Quantile regression is not the only way to get such continuous-valued intervals.
However, it is often the best way, especially if $\alpha$ is known in advance.
The reason is that the intervals generated via quantile regression even without conformal prediction, i.e. $[\hat{t}_{\alpha/2}(x),\hat{t}_{1-\alpha/2}(x)]$, have good coverage to begin with.
Furthermore, they have asymptotically valid conditional coverage (a concept we will explain in Section~\ref{sec:evaluating}).
These properties propagate through the conformal procedure and lead to prediction sets with good performance.
 
One attractive feature of quantile regression is that it can easily be added on top of any base model simply by changing the loss function to a \emph{quantile loss} (informally referred to as a \emph{pinball loss}),
\begin{figure}[ht]
    \begin{minipage}{0.7\textwidth}
    \begin{equation}
        \label{eq:pinball-loss}
        L_{\gamma}(\hat{t}_{\gamma},y) =  (y-\hat{t}_{\gamma})\gamma \ind{y > \hat{t}_{\gamma}} + (\hat{t}_{\gamma}-y)(1-\gamma) \ind{y \leq \hat{t}_{\gamma}}. 
    \end{equation}
    \end{minipage}
    \begin{minipage}{0.3\textwidth}
        \centering
        \includegraphics[width=\textwidth]{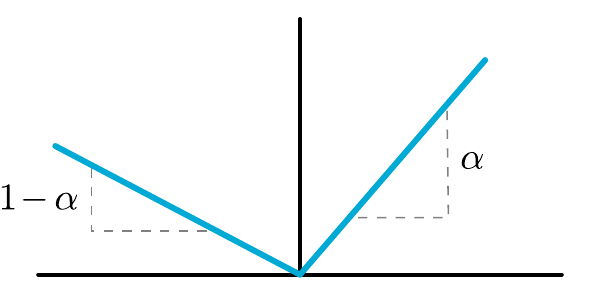}
    \end{minipage}
\end{figure}

\noindent The reader can think of quantile regression as a generalization of L1-norm regression: when $\gamma=0.5$, the loss function reduces to $L_{0.5}=|\hat{t}_{\gamma}(x)-y|/2$, which encourages $\hat{t}_{0.5}(x)$ to converge to the conditional median.
Changing $\gamma$ just modifies the L1 norm as in the illustration above to target other quantiles.
In practice, one can just use a quantile loss instead of MSE at the end of any algorithm, like a neural network, in order to regress to a quantile.

\subsection{Conformalizing Scalar Uncertainty Estimates}
\label{subsec:uncertainty-scalar}
\subsubsection{The Estimated Standard Deviation}
As an alternative to quantile regression, our next example is a different way of constructing prediction sets for continuous $y$ with a less rich but more common notion of heuristic uncertainty: an estimate of the standard deviation $\hat{\sigma}(x)$.  For example, one can produce uncertainty scalars by assuming $Y_{\rm test}\mid X_{\rm test} = x$ follows some parametric distribution---like a Gaussian distribution---and training a model to output the mean and variance of that distribution.
To be precise, in this setting we choose to model $Y_{\rm test}\mid X_{\rm test} = x \sim \mathcal{N}(\mu(x),\sigma(x))$, and we have models $\hat{f}(x)$ and $\hat{\sigma}(x)$ trained to maximize the likelihood of the data with respect to $\E\left[Y_{\rm test}\mid X_{\rm test = x}\right]$ and $\sqrt{\textrm{Var} \left[ Y_{\rm test}\mid X_{\rm test} = x \right]}$ respectively.
Then, $\hat{f}(x)$ gets used as the point prediction and $\hat{\sigma}(x)$ gets used as the uncertainty.
This strategy is so common that it is commoditized: there are inbuilt PyTorch losses, such as \texttt{GaussianNLLLoss}, that enable training a neural network this way.
However, we usually know $Y_{\rm test}\mid X_{\rm test}$ isn't Gaussian, so even if we had infinite data, $\hat{\sigma}(x)$ would not necessarily be reliable. 
We can use conformal prediction to turn this heuristic uncertainty notion into rigorous prediction intervals of the form $\hat{f}(x) \pm \hat{q}\hat{\sigma}(x)$.

\subsubsection{Other 1-D Uncertainty Estimates}
More generally, we assume there is a function $u(x)$ such that larger values encode more uncertainty.
This single number can have many interpretations beyond the standard deviation.
For example, one instance of an uncertainty scalar simply involves the user creating a model for the magnitude of the residual.
In that setting, the user would first fit a model $\hat{f}$ that predicts $y$ from $x$.
Then, they would fit a second model $\hat{r}$ (possibly the same neural network), that predicts $\left|y-\hat{f}(x)\right|$.
If $\hat{r}$ were perfect, we would expect the set $\left[\hat{f}(x)-\hat{r}(x), \hat{f}(x)+\hat{r}(x)\right]$ to have perfect coverage.
However, our learned model of the error $\hat{r}$ is often poor in practice.

There are many more such uncertainty scalars than we can discuss in this document in detail, including 
\begin{enumerate}
    \item measuring the variance of $\hat{f}(x)$ across an ensemble of models,
    \item measuring the variance of $\hat{f}(x)$ when randomly dropping out a fraction of nodes in a neural net,
    \item measuring the variance of $\hat{f}(x)$ to small, random input perturbations,
    \item measuring the variance of $\hat{f}(x)$ over different noise samples input to a generative model,
    \item measuring the magnitude of change in $\hat{f}(x)$ when applying an adversarial perturbation, etc.
\end{enumerate}
These cases will all be treated the same way.
There will be some point prediction $\hat{f}(x)$, and some uncertainty scalar $u(x)$ that is large when the model is uncertain and small otherwise (in the residual setting, $u(x):=\hat{r}(x)$, and in the Gaussian setting, $u(x):=\hat{\sigma}(x)$).
We will proceed with this notation for the sake of generality, but the reader should understand that $u$ can be replaced with any function.

Now that we have our heuristic notion of uncertainty in hand, we can define a score function,
\begin{equation}
\label{eq:uncertainty_scalar_score}
    s(x,y) = \frac{\left|y - \hat{f}(x)\right|}{u(x)}.
\end{equation}
This score function has a natural interpretation: it is a multiplicative correction factor of the uncertainty scalar (i.e., $s(x,y)u(x)=\left|y-\hat{f}(x)\right|$).
As before, taking $\hat{q}$ to be the $\frac{\lceil (1-\alpha)(n+1) \rceil}{n}$ quantile of the calibration scores guarantees us that for a new example,
\begin{equation}
    \P\left[s(X_{\rm test},Y_{\rm test}) \leq \hat{q} \right] \geq 1-\alpha \implies \P\left[\left|Y_{\rm test}-\hat{f}(X_{\rm test})\right| \leq u(X_{\rm test})\hat{q} \right] \geq 1-\alpha.
\end{equation}
Naturally, we can then form prediction sets using the rule
\begin{equation}
    \C(x) = \left[ \hat{f}(x) - u(x)\hat{q}, \hat{f}(x) + u(x)\hat{q} \right].
    \label{eq:us-sets}
\end{equation}

\begin{figure}[t]
    \centering
    \begin{minted}[fontsize=\footnotesize]{python}
# model(X)[:,0]=E(Y|X), and model(X)[:,1]=stddev(Y|X)
scores = abs(model(calib_X)[:,0]-calib_Y)/model(calib_X)[:,1]
# Get the score quantile
qhat = torch.quantile(scores,np.ceil((n+1)*(1-alpha))/n)
# Deploy (represent sets as tuple of lower and upper endpoints)
muhat, stdhat = (model(test_X)[:,0], model(test_X)[:,1])
prediction_sets = (muhat-stdhat*qhat, muhat+stdhat*qhat)
    \end{minted}
    \caption{\textbf{Python code for conformalized uncertainty scalars.} \jupyter{https://github.com/aangelopoulos/conformal-prediction/blob/main/notebooks/meps-uncertainty-scalar.ipynb}}
    \label{fig:uncertainty-scalar-code}
\end{figure}

\begin{figure}[h]
    \centering
    \includegraphics[width=0.3\textwidth]{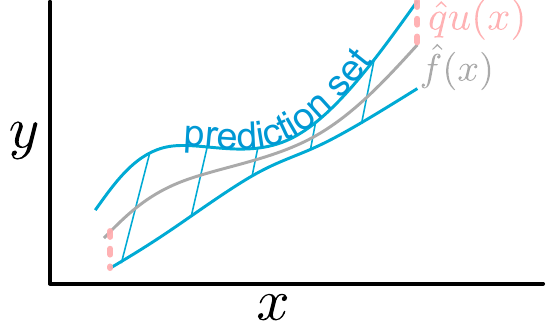}
    \caption{\textbf{A visualization of the uncertainty scalars algorithm} in Eq.~\eqref{eq:us-sets}.
    We produce the set by adding and subtracting $\hat{q}u(x)$.
    The constant $\hat{q}$ is picked during the calibration step.}
    \label{fig:us-diagram}
\end{figure}

Let's reflect a bit on the nature of these prediction sets.
The prediction sets are valid, as we desired.
Due to our construction, they are also symmetric about the prediction, $\hat{f}(x)$, although symmetry could be relaxed with minor modifications.
However, uncertainty scalars do not necessarily scale properly with $\alpha$.
In other words, there is no reason to believe that a quantity like $\hat{\sigma}$ would be directly related to quantiles of the label distribution.
We tend to prefer quantile regression when possible, since it directly estimates this quantity and thus should be a better heuristic (and in practice it usually is; see~\cite{angelopoulos2022image} for some evaluations).
Nonetheless, uncertainty scalars remain in use because they are easy to deploy and have been commoditized in popular machine learning libraries.
See Figure~\ref{fig:uncertainty-scalar-code} for a Python implementation of this method.

\subsection{Conformalizing Bayes}
Our final example of conformal prediction will use a Bayesian model.
Bayesian predictors, like Bayesian neural networks, are commonly studied in the field of uncertainty quantification, but rely on many unverifiable and/or incorrect assumptions to provide coverage.
Nonetheless, we should incorporate any prior information we have into our prediction sets.
We will now show how to create valid prediction sets that are also Bayes optimal among all prediction sets that achieve $1-\alpha$ coverage.
These prediction sets use the posterior predictive density as a conformal score.
The Bayes optimality of this procedure was first proven in~\cite{hoff2021bayes}, and was previously studied in~\cite{wasserman2011frasian,melluish2001comparing}.
Because our algorithm reduces to picking the labels with high posterior predictive density, the Python code will look exactly the same as in Figure~\ref{fig:teaser}.
The only difference is interpretation, since the softmax now represents an approximation of a continuous distribution rather than a categorical one.

Let us first describe what a Bayesian would do, given a Bayesian model $\hat{f}(y \mid x)$, which estimates the value of the posterior distribution of $Y_{\rm test}$ at label $y$ with input $X_{\rm test}=x$.
If one believed all the necessary assumptions---mainly, a correctly specified model and asymptotically large $n$---the following would be the optimal prediction set:
\begin{equation}
    S(x)=\left\{y : \hat{f}(y \mid x) > t \right\}\text{, where }t\text{ is chosen so }\int\limits_{y \in S(x)} \hat{f}(y \mid x) dy = 1-\alpha.
\end{equation}
However, because we cannot make assumptions on the model and data, we can only consider $\hat{f}(y \mid x)$ to be a heuristic notion of uncertainty.

Following our now-familiar checklist, we can define a conformal score,
\begin{equation}
    \label{eq:bayes-score}
    s(x,y) = -\hat{f}(y \mid x),
\end{equation}
which is high when the model is uncertain and otherwise low.
After computing $\hat{q}$ over the calibration data, we can then construct prediction sets:
\begin{equation}
    \label{eq:bayes-sets}
    \C(x) = \left\{ y : \hat{f}(y \mid x) > -\hat{q} \right\}.
\end{equation}

\begin{figure}[h]
    \centering
    \includegraphics[width=0.3\textwidth]{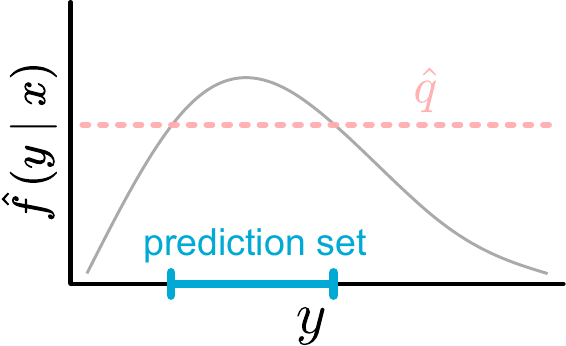}
    \caption{\textbf{A visualization of the conformalized Bayes algorithm} in Eq.~\eqref{eq:bayes-sets}.
    The prediction set is a superlevel set of the posterior predictive density.}
    \label{fig:bayes-diagram}
\end{figure}

This set is valid because we chose the threshold $\hat{q}$ via conformal prediction.
Furthermore, when certain technical assumptions are satisfied, it has the best Bayes risk among all prediction sets with $1-\alpha$ coverage.
To be more precise, under the assumptions in~\cite{hoff2021bayes}, $\C(X_{\rm test})$ has the smallest average size of any conformal procedure with $1-\alpha$ coverage, where the average is taken over the data \emph{and} the parameters.
This result should not be a surprise to those familiar with decision theory, as the argument we are making feels similar to that of the Neyman-Pearson lemma.
This concludes the final example.

\subsubsection*{Discussion}

As our examples have shown, conformal prediction is a simple and pragmatic technique with many use cases.
It is also easy to implement and computationally trivial.
Additionally, the above four examples serve as roadmaps to the user for designing score functions with various notions of optimality, including average size, adaptivity, and Bayes risk. 
Still more is yet to come---conformal prediction can be applied more broadly than it may first seem at this point.
We will outline extensions of conformal prediction to other prediction tasks such as outlier detection, image segmentation, serial time-series prediction, and so on in Section~\ref{sec:advanced}.
Before addressing these extensions, we will take a deep dive into diagnostics for conformal prediction in the standard setting, including the important topic of conditional coverage.

\section{Evaluating Conformal Prediction}
\label{sec:evaluating}

We have spent the last two sections learning how to form valid prediction sets satisfying rigorous statistical guarantees.
Now we will discuss how to evaluate them.
Our evaluations will fall into one of two categories.
\begin{enumerate}
    
    \item \textbf{Evaluating adaptivity.} It is extremely important to keep in mind that the conformal prediction procedure with the smallest average set size is not necessarily the best.
    A good conformal prediction procedure will give small sets on easy inputs and large sets on hard inputs in a way that faithfully reflects the model's uncertainty.  This \emph{adaptivity} is not implied by conformal prediction's coverage guarantee, but it is  non-negotiable in practical deployments of conformal prediction. We will formalize adaptivity, explore its consequences, and suggest practical algorithms for evaluating it.

    \item \textbf{Correctness checks.} Correctness checks help you test whether you've implemented conformal prediction correctly. We will empirically check that the coverage satisfies Theorem~\ref{thm:conformal_calibration}. Rigorously evaluating whether this property holds requires a careful accounting of the finite-sample variability present with real datasets. We develop explicit formulae for the size of the benign fluctuations---if one observes deviations from $1-\alpha$ in coverage that are larger than these formulae dictate, then there is a problem with the implementation.
 
\end{enumerate}

Many of the evaluations we suggest are computationally intensive, and require running the entire conformal procedure on different splits of data at least $100$ times.
Na\"ive implementations of these evaluations can be slow when the score takes a long time to compute.
With some simple computational tricks and strategic caching, we can speed this process up by orders of magnitude.
Therefore to aid the reader, we intersperse the mathematical descriptions with code to efficiently implement these computations.

\subsection{Evaluating Adaptivity}
Although any conformal procedure yields prediction intervals that satisfy~\eqref{eq:coverage}, there are many such procedures, and they differ in other important ways. In particular, a key design consideration for conformal prediction is \emph{adaptivity}: we want the procedure to return larger sets for harder inputs and smaller sets for easier inputs. While most reasonable conformal procedures will satisfy this to some extent, we now discuss precise metrics for adaptivity that allow the user to check a conformal procedure and to compare multiple alternative conformal procedures. 
\paragraph{Set size.} 
The first step is to plot histograms of set sizes.
This histogram helps us in two ways.
Firstly, a large average set size indicates the conformal procedure is not very precise, indicating a possible problem with the score or underlying model.
Secondly, the spread of the set sizes shows whether the prediction sets properly adapt to the difficulty of examples. A wider spread is generally desirable, since it means that the procedure is effectively distinguishing between easy and hard inputs.
\begin{figure}[H]
    \centering
    \includegraphics[width=0.6\linewidth]{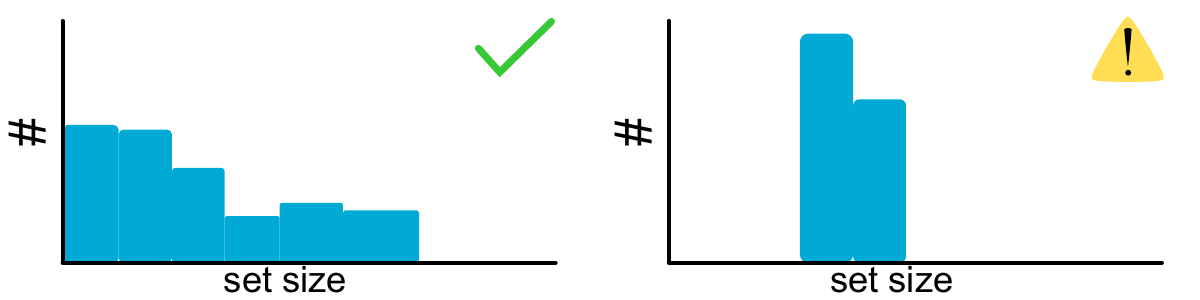}
\end{figure}

It can be tempting to stop evaluations after plotting the coverage and set size, but certain important questions remain unanswered.
A good spread of set sizes is generally better, but it does not necessarily indicate that the sets adapt properly to the difficulty of $X$. Above seeing that the set sizes have dynamic range, we will need to verify that large sets occur for hard examples. We next formalize this notion and give metrics for evaluating it.

\paragraph{Conditional coverage.} 
Adaptivity is typically formalized by asking for the \emph{conditional coverage} \cite{vovk2012conditional} property:
\begin{equation}
    \P\left[ Y_{\rm test} \in \C(X_{\rm test}) \mid X_{\rm test} \right] \geq 1-\alpha.
    \label{eq:conditional-coverage}
\end{equation}
That is, for every value of the input $X_{\rm test}$, we seek to return prediction sets with $1-\alpha$ coverage. 
This is a stronger property than the \emph{marginal coverage} property in $\eqref{eq:coverage}$ that conformal prediction is guaranteed to achieve---indeed, in the most general case, conditional coverage is impossible to achieve~\cite{vovk2012conditional}. In other words, conformal procedures are not guaranteed to satisfy~\eqref{eq:conditional-coverage}, so we must check how close our procedure comes to approximating it.

The difference between marginal and conditional coverage is subtle but of great practical importance, so we will spend some time think about the differences here. 
Imagine there are two groups of people, group A and group B, with frequencies 90\% and 10\%.
The prediction sets always cover $Y$ among people in group A and never cover $Y$ when the person comes from group B.
Then the prediction sets have 90\% coverage, but not conditional coverage.
Conditional coverage would imply that the prediction sets cover $Y$ at least 90\% of the time in both groups.
This is necessary, but not sufficient; conditional coverage is a very strong property that states the probability of the prediction set needs to be $\geq 90\%$ \emph{for a particular person}.
In other words, for any subset of the population, the coverage should be $\geq 90\%$. 
See Figure~\ref{fig:conditional-marginal} for a visualization of the difference between conditional and marginal coverage.
\begin{figure}[H]
    \centering
    \includegraphics[width=0.9\linewidth]{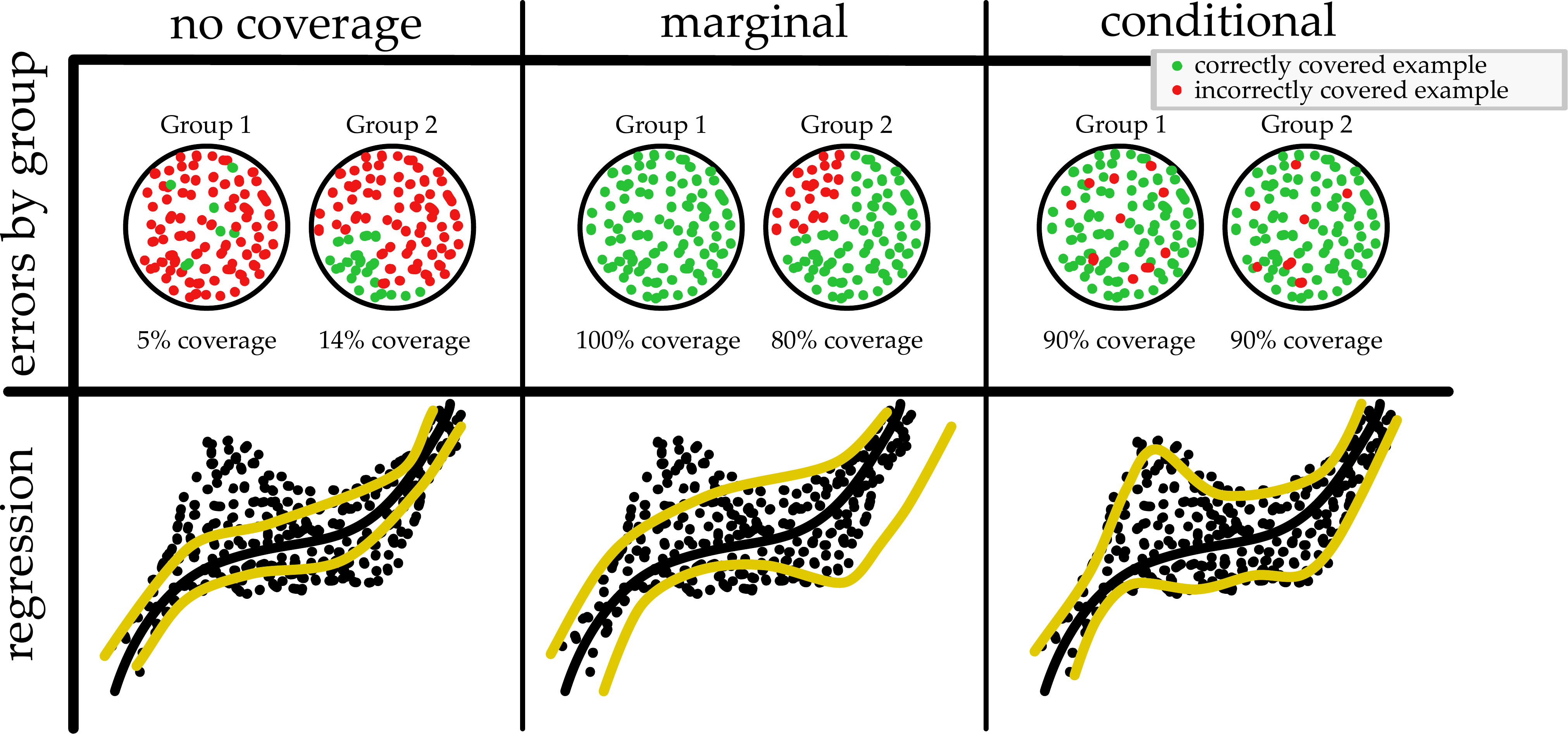}
    \caption{\textbf{Prediction sets with various notions of coverage:} no coverage, marginal  coverage, or conditional coverage (at a level of 90\%).  In the marginal case, all the errors happen in the same groups and regions in $X$-space.  Conditional coverage disallows this behavior, and errors are evenly distributed.
    }
    \label{fig:conditional-marginal}
\end{figure}

\paragraph{Feature-stratified coverage metric.} 
As a first metric for conditional coverage, we will formalize the example we gave earlier, where coverage is unequal over some groups.
The reader can think of these groups as discrete categories, like race, or as a discretization of continuous features, like age ranges.
Formally, suppose we have features $X_{i,1}^{\rm (val)}$ that take values in $\{1,\dots,G\}$ for some $G$.
(Here, $i = 1,\dots,n_{\rm val}$ indexes the example in the validation set, and the first coordinate of each feature is the group.)
Let $\Ical_g \subset \{1,\dots,n_{\rm val}\}$ be the set of observations such that $X_{i,1}^{\rm (val)} = g$ for $g = 1,\dots,G$. 
Since conditional coverage implies that the procedure has the same coverage for all values of $X_{\rm test}$, we use the following measure:
\begin{equation*}
    \textbf{FSC metric}: \qquad \min_{g \in \{1,\dots,G\}} \  \frac{1}{|\Ical_g|} \ \sum_{i\in\Ical_g} \ind{Y_i^{\rm (val)} \in \C\Big(X_i^{\rm (val)}\Big)}
\end{equation*}
In words, this is the observed coverage among all instances where the discrete feature takes the value $g$.
If conditional coverage were achieved, this would be $1-\alpha$, and values farther below $1-\alpha$ indicate a greater violation of conditional coverage.
Note that this metric can also be used with a continuous feature by binning the features into a finite number of categories.

\paragraph{Size-stratified coverage metric.} We next consider a more general-purpose metric for how close a conformal procedure comes to satisfying~\eqref{eq:conditional-coverage}, introduced in \citep{angelopoulos2020sets}. First, we discretize the possible cardinalities of $\C(x)$, into $G$ bins, $B_1,\dots,B_G$. For example, in classification we might divide the observations into three groups, depending on whether $\C(x)$ has one element, two elements, or more than two elements. Let $\Ical_g \subset \{1,\dots,n_{\rm val}\}$ be the set of observations falling in bin $g$ for $g = 1,\dots,G$. Then we consider the following 
\begin{equation*}
    \textbf{SSC metric}: \qquad \min_{g \in \{1,\dots,G\}} \  \frac{1}{|\Ical_g|} \ \sum_{i\in\Ical_g} \ind{Y_i^{\rm (val)} \in \C\Big(X_i^{\rm (val)}\Big)}
\end{equation*}
In words, this is the observed coverage for all units for which the set size $|\C(x)|$ falls into bin $g$. As before, if conditional coverage were achieved, this would be $1-\alpha$, and values farther below $1-\alpha$ indicate a greater violation of conditional coverage. Note that this is the same expression as for the FSC metric, except that the definition of $\Ical_g$ has changed.
Unlike the FSC metric, the user does not have to define an important set of discrete features a-priori---it is a general metric that can apply to any example.

See \cite{cauchois2020knowing} and \cite{feldman2021improving} for additional metrics of conditional coverage.

\subsection{The Effect of the Size of the Calibration Set}
\label{subsec:cal_size}
We first pause to discuss how the size of the calibration set affects conformal prediction. 
We consider this question for two reasons. 
First, the user must choose this for a practical deployment. Roughly speaking, our conclusion will that be choosing a calibration set of size $n=1000$ is sufficient for most purposes.
Second, the size of the calibration set is one source of finite-sample variability that we will need to analyze to correctly check the coverage. 
We will build on the results here in the next section, where we give a complete description of how to check coverage in practice. 

How does the size of the calibration set, $n$, affect conformal prediction? 
The coverage guarantee in~\eqref{eq:coverage} holds for any $n$, so we can see that our prediction sets have coverage at least $1-\alpha$ even with a very small calibration set. 
Intuitively, however, it may seem that larger $n$ is better, and leads to more stable procedures. This intuition is correct, and it explains why using a larger calibration set is beneficial in practice. The details are subtle, so we carefully work through them here.

The key idea is that \emph{the coverage of conformal prediction conditionally on the calibration set is a random quantity}. 
That is, if we run the conformal prediction algorithm twice, each time sampling a new calibration dataset, then check the coverage on an infinite number of validation points, those two numbers will not be equal.
The coverage property in~\eqref{eq:coverage} says that coverage will be at least $1-\alpha$ on average over the randomness in the calibration set, but with any one fixed calibration set, the coverage on an infinite validation set will be some number that is not exactly $1-\alpha$. Nonetheless, we can choose $n$ large enough to control these fluctuations in coverage by analyzing its distribution.

In particular, the distribution of coverage has an analytic form, first introduced by Vladimir Vovk in~\cite{vovk2012conditional}, namely, 
\begin{equation}
    \label{eq:beta}
     \P\left(Y_{\rm test} \in \C\left(X_{\rm test}\right) \big| \: \{(X_i,Y_i)\}_{i=1}^n\right) \sim \textrm{Beta}\left(n+1-l,l \right),
\end{equation}
where
\begin{equation*}
    l=\lfloor(n+1)\alpha\rfloor.
\end{equation*}
Notice that the conditional expectation above is the coverage with an infinite validation data set, holding the calibration data fixed. 
A simple proof of this fact is available in~\cite{vovk2012conditional}.
We plot the distribution of coverage for several values of $n$ in Figure~\ref{fig:beta}.

\begin{figure}[t]
    \centering
    \includegraphics[width=0.6\linewidth]{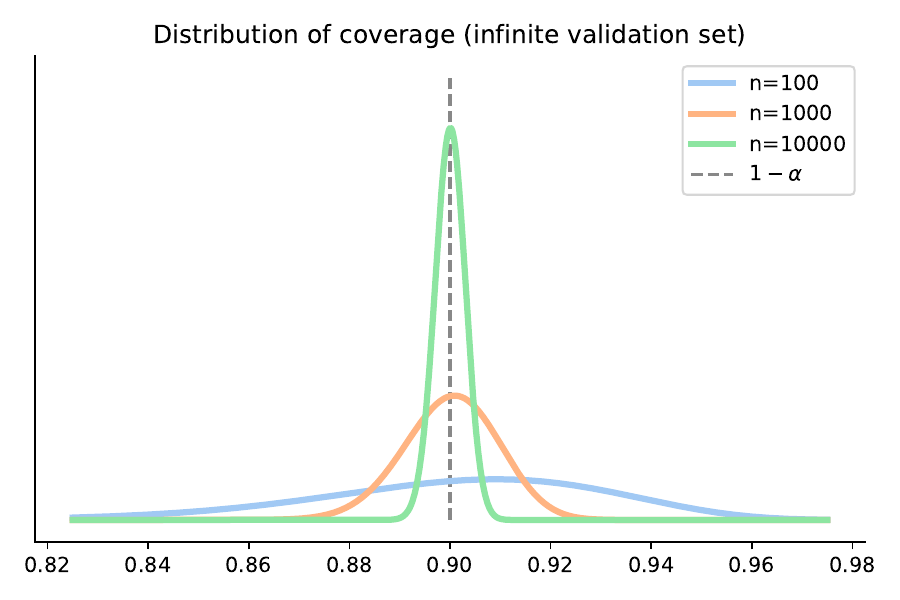}
    \caption{\textbf{The distribution of coverage} with an infinite validation set is plotted for different values of $n$ with $\alpha=0.1$. The distribution converges to $1-\alpha$ with rate $\mathcal{O}\left(n^{-1/2}\right)$.}
    \label{fig:beta}
\end{figure}

Inspecting Figure~\ref{fig:beta}, we see that choosing $n=1000$ calibration points leads to coverage that is typically between $.88$ and $.92$, hence our rough guideline of choosing about $1000$ calibration points. More formally, we can compute exactly the number of calibration points $n$ needed to achieve a coverage of $1-\alpha \pm \epsilon$ with probability $1-\delta$. 
Again, the average coverage is always at least $1-\alpha$; the parameter $\delta$ controls the tail probabilities of the coverage conditionally on the calibration data.
For any $\delta$, the required calibration set size $n$ can be explicitly computed from a simple expression, and we report on several values in Table~\ref{tab:calib-set-size} for the reader's reference. 
Code allowing the user to produce results for any choice of $n$ and $\alpha$ accompanies the table.

\begin{table}[t]
    \centering
    \begin{tabular}{c|ccccc}
         $\mathbf{\epsilon}$ & 0.1 & 0.05 & 0.01 & 0.005 & 0.001 \\
         $n(\epsilon)$ & 22 & 102 & 2491 & 9812 & 244390
    \end{tabular}
    \caption{\textbf{Calibration set size $n(\epsilon)$ required} for coverage slack $\epsilon$ with $\delta=0.1$ and $\alpha=0.1$. \jupyter{https://github.com/aangelopoulos/conformal-prediction/blob/main/notebooks/correctness_checks.ipynb}}
    \label{tab:calib-set-size}
\end{table}

\subsection{Checking for Correct Coverage}

As an obvious diagnostic, the user will want to assess whether the conformal procedure has the correct coverage.
This can be accomplished by running the procedure over $R$ trials with new calibration and validation sets, and then calculating the empirical coverage for each,
\begin{equation}
    \label{eq:empirical-average-coverage}
    C_j = \frac{1}{n_{\textnormal{val}}}\sum\limits_{i=1}^{n_{\textnormal{val}}} \ind{Y^{(\text{val})}_{i,j} \in \C_j\left(X^{(\text{val})}_{i,j}\right)}\text{, for }j=1,...,R,
\end{equation}
where $n_{\textnormal{val}}$ is the size of the validation set, $(X^{(\text{val})}_{i,j},Y^{(\text{val})}_{i,j})$ is the $i$th validation example in trial $j$, and $\C_j$ is calibrated using the calibration data from the $j$th trial.
A histogram of the $C_j$ should be centered at roughly $1-\alpha$, as in Figure~\ref{fig:beta}.
Likewise, the mean value,
\begin{equation}
    \label{eq:average-empirical-coverage}
    \overline{C} = \frac{1}{R}\sum\limits_{j=1}^R C_j,
\end{equation}
should be approximately $1-\alpha$.

With real datasets, we only have $n+n_{\textnormal{val}}$ data points total to evaluate our conformal algorithm and therefore cannot draw new data for each of the $R$ rounds.
So, we compute the coverage values by randomly splitting the $n+n_{\textnormal{val}}$ data points $R$ times into calibration and validation datasets, then running conformal.
Notice that rather than splitting the data points themselves many times, we can instead first cache all conformal scores and then compute the coverage values over many random splits, as in the code sample in Figure~\ref{fig:cache-scores-code}. 

\begin{figure}[t]
    \centering
    \begin{minted}[fontsize=\footnotesize]{python}
try: # try loading the scores first
  scores = np.load('scores.npy')
except:
  # X and Y have n + n_val rows each
  scores = get_scores(X,Y)
  np.save(scores, 'scores.npy')
# calculate the coverage R times and store in list
coverages = np.zeros((R,))
for r in range(R):
  np.random.shuffle(scores) # shuffle
  calib_scores, val_scores = (scores[:n],scores[n:]) # split
  qhat = np.quantile(calib_scores, np.ceil((n+1)*(1-alpha)/n), method='higher') # calibrate
  coverages[r] = (val_scores <= qhat).astype(float).mean() # see caption
average_coverage = coverages.mean() # should be close to 1-alpha
plt.hist(coverages) # should be roughly centered at 1-alpha
    \end{minted}
    \caption{\textbf{Python code for computing coverage} with efficient score caching. Notice that from the expression for conformal sets in~\eqref{eq:conformal_set}, a validation point is covered if and only if $s(X,Y) \le \hat{q}$, which is how the third to last line is succinctly computing the coverage. \jupyter{https://github.com/aangelopoulos/conformal-prediction/blob/main/notebooks/correctness_checks.ipynb}}
    \label{fig:cache-scores-code}
    \vspace{-0.3cm}
\end{figure}

If properly implemented, conformal prediction is guaranteed to satisfy the inequality in~\eqref{eq:coverage}. 
However, if the reader sees minor fluctuations in the observed coverage, they may not need to worry: the finiteness of $n$, $n_{\textnormal{val}}$, and $R$ can lead to benign fluctuations in coverage which add some width to the Beta distribution in Figure~\ref{fig:beta}.
Appendix~\ref{app:empirical-coverage} gives exact theory for analyzing the mean and standard deviation of $\overline{C}$.
From this, we will be able to tell if any deviation from $1-\alpha$ indicates a problem with the implementation, or if it is benign. 
Code for checking the coverage at all different values of $n$, $n_{\textnormal{val}}$, and $R$ is available in the accompanying Jupyter notebook of Figure~\ref{fig:cache-scores-code}.

\section{Extensions of Conformal Prediction}
\label{sec:advanced}

At this point, we have seen the core of the matter: how to construct prediction sets with coverage in any standard supervised prediction problem.
We now broaden our horizons towards prediction tasks with different structure, such as side information, covariate shift, and so on.
These more exotic problems arise quite frequently in the real world, so we present practical conformal algorithms to address them.

\subsection{Group-Balanced Conformal Prediction}
\label{sec:group-balanced}

In certain settings, we might want prediction intervals that have equal error rates across certain subsets of the data. 
For example, we may require our medical classifier to have coverage that is correct for all racial and ethnic groups. 
To formalize this, we suppose that the first feature of our inputs, $X_{i,1}$, $i = 1, ..., n$ takes values in some discrete set $\{1,...,G\}$ corresponding to categorical groups.
We then ask for \emph{group-balanced} coverage:
\begin{equation}
    \label{eq:group-balanced-coverage}
    \P\left( Y_{\rm test} \in \C(X_{\rm test}) \; \rvert \; X_{\rm test,1} = g \right) \ge 1-\alpha,
\end{equation}
for all groups $g \in \{1,\dots,G\}$.
In words, this means we have a $1-\alpha$ coverage rate for all groups.
Notice that the group output could be a post-processing of the original features in the data.
For example, we might bin the values of $X_{\rm test}$ into a discrete set.

Recall that a standard application of conformal prediction will not necessarily yield coverage within each group simultaneously---that is,~\eqref{eq:group-balanced-coverage} may not be satisfied.
We saw an example in Figure~\ref{fig:conditional-marginal}; the marginal guarantee from normal conformal prediction can still be satisfied even if all errors happen in one group.

In order to achieve group-balanced coverage, we will
simply run conformal prediction seperately for each
group, as visualized below.
\begin{figure}[H]
    \centering
    \includegraphics[width=0.6\linewidth]{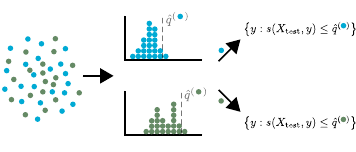}
\end{figure}
Making this formal, given a conformal score function $s$, we stratify the scores on the calibration set by group,
\begin{equation}
    s_i^{(g)} = s(X_j,Y_j) \text{, where } X_{j,1} \text{ is the } i \text{th occurrence of group } g.
\end{equation}
Then, within each group, we calculate the conformal quantile
\begin{equation}
    \hat{q}^{(g)} = \quantile\left(s_1,...,s_{n^{(g)}}; \frac{\big\lceil (n^{(g)}+1)(1-\alpha) \big\rceil}{n^{(g)}} \right)\text{, where } n^{(g)}\text{ is the number of examples of group } g.
\end{equation}
Finally, we form prediction sets by first picking the relevant quantile,
\begin{equation}
    \C(x) = \left\{ y : s(x,y) \leq \hat{q}^{(x_1)} \right\}.
\end{equation}
That is, for a point $x$ that we see falls in group $x_1$, we use the threshold $\hat{q}^{(x_1)}$ to form the prediction set, and so on.
This choice of $\C$ satisfies~\eqref{eq:group-balanced-coverage}, as was first documented by Vovk in~\cite{vovk2012conditional}.
\begin{prop}[Error control guarantee for group-balanced conformal prediction]
Suppose $(X_1,Y_1),\dots,\\(X_n,Y_n), (X_{\text{test}},Y_{test})$ are an i.i.d. sample from some distribution.
Then the set $\C$ defined above satisfies the error control property in~\eqref{eq:group-balanced-coverage}.
\end{prop}


\subsection{Class-Conditional Conformal Prediction}
In classification problems, we might similarly ask for coverage on \emph{every} ground truth class. 
For example, if we had a medical classifier assigning inputs to class normal or class cancer, we might ask that the prediction sets are 95\% accurate both when the ground truth is class cancer and also when the ground truth is class normal.
Formally, we return to the classification setting, where $\Y = \{1,...,K\}$. 
We seek to achieve \emph{class-balanced} coverage,
\begin{equation}
    \label{eq:class-balanced-coverage}
    \P\left( Y_{\rm test} \in \C(X_{\rm test}) \; \rvert \; Y_{\rm test} = y \right) \geq 1-\alpha,
\end{equation}
for all classes $y \in \{1,\dots,K\}$.

To achieve class-balanced coverage, we will calibrate within each class separately.
The algorithm will be similar to the group-balanced coverage of Section~\ref{sec:group-balanced}, but we must modify it because we do not know the correct class at test time. 
(In contrast, in Section~\ref{sec:group-balanced}, we observed the group information $X_{\mathrm{test},1}$ as an input feature.) 
See the visualization below.
\begin{figure}[H]
    \centering
    \includegraphics[width=0.6\linewidth]{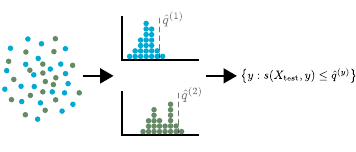}
\end{figure}
Turning to the algorithm, given a conformal score function $s$, stratify the scores on the calibration set by class,
\begin{equation}
    s_i^{(k)} = s(X_j,Y_j) \text{, where } Y_j \text{ is the } i \text{th occurrence of class } k.
\end{equation}
Then, within each class, we calculate the conformal quantile,
\begin{equation}
    \hat{q}^{(k)} = \quantile\left(s_1,...,s_{n^{(k)}}; \frac{\big\lceil (n^{(k)}+1)(1-\alpha) \big\rceil}{n^{(k)}} \right)\text{, where } n^{(k)}\text{ is the number of examples of class } k.
\end{equation}
Finally, we iterate through our classes and include them in the prediction set based on their quantiles:
\begin{equation}
    \C(x) = \left\{ y : s(x,y) \leq \hat{q}^{(y)} \right\}.
\end{equation}
Notice that in the preceding display, we take a provisional value of the response, $y$,
and then use the conformal threshold $\hat{q}^{(y)}$ to
determine if it is included in the prediction set.
This choice of $\C$ satisfies~\eqref{eq:class-balanced-coverage}, as proven by Vovk in~\cite{vovk2012conditional}; another version can be found in~\cite{Sadinle2016LeastAS}.
\begin{prop}[Error control guarantee for class-balanced conformal prediction]
Suppose $(X_1,Y_1),\dots,\\(X_n,Y_n),(X_{\text{test}},Y_{test})$ are an i.i.d. sample from some distribution.
Then the set $\C$ defined above satisfies the error control property in~\eqref{eq:class-balanced-coverage}.
\end{prop}

\subsection{Conformal Risk Control}
\label{subsec:conformal-risk}
So far, we have used conformal prediction to construct prediction sets that bound the \emph{miscoverage},\looseness=-1
\begin{equation}
    \label{eq:miscoverage}
    \P\Big( Y_{\rm test} \notin \C(X_{\rm test}) \Big) \leq \alpha.
\end{equation}
However, for many machine learning problems, the natural notion of error is not miscoverage.
Here we show that conformal prediction can also provide guarantees of the form
\begin{equation}
    \label{eq:risk-upper-bound}
    \E\Big[ \ell\big(\C(X_{\rm test}), Y_{\rm test}\big) \Big] \leq \alpha,
\end{equation}
for any bounded \emph{loss function} $\ell$ that shrinks as $\C$ grows.
This is called a \emph{conformal risk control} guarantee.
Note that~\eqref{eq:risk-upper-bound} recovers~\eqref{eq:miscoverage} when using the miscoverage loss, $\ell\big(C(X_{\rm test}),Y_{\rm test}\big) = \ind{Y_{\rm test} \notin C(X_{\rm test})}$.
However, this algorithm also  extends conformal prediction to situations where other loss functions, such as the false negative rate (FNR), are more appropriate.

As an example, consider multilabel classification. Here, the response $Y_i \subseteq \{1,...,K\}$ a subset of $K$ classes.
Given a trained model $f : \X \to [0,1]^K$, we wish to output sets that include a large fraction of the true classes in $Y_i$. 
To that end, we post-process the model's raw outputs into the set of classes with sufficiently high scores, $\Clam(x) = \{ k : f(X)_k \geq 1- \lambda \}$.
Note that as the threshold $\lambda$ grows, we include more classes in $\Clam(x)$---it becomes more conservative in that we are less likely to omit true classes.
Conformal risk control can be used to find a threshold value $\lhat$ that controls the fraction of missed classes. That is, $\lhat$ can be chosen so that the expected value of $\ell\big( \Clamhat(X_{\rm test}), Y_{\rm test} \big) = 1 - |Y_{\rm test} \cap \Clam(X_{\rm test})|/ |Y_{\rm test}|$ is guaranteed to fall below a user-specified error rate $\alpha$.
For example, setting $\alpha=0.1$ ensures that $\Clamhat(X_{\rm test})$ contains $90\%$ of the true classes in $Y_{\rm test}$ on average. We will work through a multilabel classification example in detail in Section~\ref{subsec:multilabel-fnr}.

Formally, we will consider post-processing the predictions of the model $f$ to create a prediction set $\Clam(\cdot)$. 
The prediction set has a parameter $\lambda$ that
encodes its level of conservativeness: larger $\lambda$ values yield more conservative outputs (e.g., larger prediction sets). 
To measure the quality of the output of $\Clam$, we consider a loss function $\ell(\Clam(x), y) \in (-\infty, B]$ for some $B<\infty$.
We require the loss function to be non-increasing as a function of $\lambda$.
The following algorithm picks $\lhat$ so that risk control as in~\eqref{eq:risk-upper-bound} holds:
\begin{equation}
\label{eq:lhat}
\lhat = \inf\left\{ \lambda : \Rhat(\lambda) \leq \alpha - \frac{B-\alpha}{n} \right\},
\end{equation}
where $\Rhat(\lambda) = \big(\ell\big(\Clam(X_1),Y_1\big) + \ldots + \ell\big(\Clam(X_n),Y_n\big)\big)/n $ is the empirical risk on the calibration data.
Note that this algorithm simply corresponds to tuning based on the empirical risk at a slightly more conservative level than $\alpha$.
For example, if $B=1$, $\alpha=0.1$, and we have $n=1000$ calibration points, then we select $\lhat$ to be the value where empirical risk hits level $\lhat=0.0991$ instead of $0.1$.
\begin{figure}[H]
    \centering
    \hspace{1.5cm}
    \includegraphics[width=0.34\linewidth]{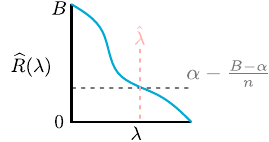}
\end{figure}
Then the prediction set $\Clamhat(X_{\rm test})$ satisfies~\eqref{eq:risk-upper-bound}.
\begin{theorem}[Conformal Risk Control~\cite{angelopoulos2022conformal}]
Suppose $(X_1,Y_1),\dots,(X_n,Y_n),(X_{\text{test}},Y_{test})$ are an i.i.d. sample from some distribution.
	Further, suppose $\ell$ is a monotone function of $\lambda$, i.e., one satisfying
	\begin{equation}
    \label{eq:risk-control-monotonicity}
		\ell\big( \C_{\lambda_1}(x), y \big) \geq \ell\big( \C_{\lambda_2}(x), y \big)
	\end{equation}
  for all $(x,y)$ and $\lambda_1 \leq \lambda_2$. Then
  \begin{equation}
		\E\left[ \ell\big(\Clamhat(X_{\rm test}), Y_{\rm test}\big) \right] \leq \alpha,
  \end{equation}
  where $\lhat$ is picked as in~\eqref{eq:lhat}.
\end{theorem}
Theory and worked examples of conformal risk control are presented in~\cite{angelopoulos2022conformal}.
In Sections~\ref{subsec:multilabel-fnr} and~\ref{subsec:tumor-fnr} we show a worked example of conformal risk control applied to tumor segmentation.
Furthermore, Appendix~\ref{app:ltt} describes a more powerful technique called Learn then Test~\cite{angelopoulos2021learn} capable of controlling general risks that do not satisfy~\eqref{eq:risk-control-monotonicity}.

\subsection{Outlier Detection}
\label{subsec:outlier}
Conformal prediction can also be adapted to handle unsupervised outlier detection. 
Here, we have access to a clean dataset $X_1,\dots,X_n$ and wish to detect when test points do not come from the same distribution. 
As before, we begin with a heuristic model that tries to identify outliers;
a larger score means that the model judges the point more likely to be an outlier.
We will then use a variant of conformal prediction to calibrate it
to have statistical guarantees. In particular, we will
guarantee that it does not return too many false positives. 

Formally, we will construct a function that labels test points as outliers or inliers, $\C : \X \to \{\text{outlier}, \text{inlier}\}$, such that
\begin{equation}
\label{eq:outlier_error_control}
    \P\left(\C(X_\text{test}) = \text{outlier}\right) \le \alpha,
\end{equation}
where the probability is over $X_{\rm test}$, a fresh sample from the clean-data distribution.
The algorithm for achieving~\eqref{eq:outlier_error_control} is similar to the 
usual conformal algorithm. We start with a conformal score $s : \mathcal{X} \to \R$ (note that 
since we are in the unsupervised setting, the score only depends on the features).
Next, we compute the conformal score on the clean data: $s_i = s(X_i)$ for $i=1,\dots,n$.
Then, we compute the conformal threshold in the usual way:
\begin{equation*}
    \hat{q} = \text{quantile}\left(s_1,\ldots ,s_n; \frac{\big\lceil (n+1)(1-\alpha) \big\rceil}{n}\right).
\end{equation*}
Lastly, when we encounter a test point, we declare it to be an outlier if the 
score exceeds $\hat{q}$:
\begin{equation*}
    \C(x) = \begin{cases}
    \text{inlier} & \text{ if } s(x) \le \hat{q} \\
    \text{outlier} & \text{ if } s(x) > \hat{q}
    \end{cases}.
\end{equation*}

This construction guarantees error control, as we record next.
\begin{prop}[Error control guarantee for outlier detection]
\label{prop:outlier}
Suppose $X_1,\dots,X_n,X_{\text{test}}$ are an i.i.d. sample from some distribution.
Then the set $\C$ defined above satisfies the error control property in~\eqref{eq:outlier_error_control}.
\end{prop}

As with standard conformal prediction, the score function is very important for
the method to perform well---that is, to be effective at flagging outliers. 
Here, we wish to 
choose the score function to effectively distinguish the type of outliers that we expect
to see in the test data from the clean data. 
The general problem of training models
to distinguish outliers is sometimes called \emph{anomaly detection}, \emph{novelty detection}, or \emph{one-class classification}, and there are good out-of-the box methods for doing this; see~\cite{pimental2014review}
for an overview of outlier detection.
Conformal outlier detection can also be seen as a hypothesis testing problem; points that are rejected as outliers have a p-value less than $alpha$ for the null hypothesis of exchangeability with the calibration data.
This interpretation is closely related to the classical permutation test~\cite{fisher1936design,pitman1937significance}.
See~\cite{vovk2003testing, guan2019prediction, bates2021multiple}
for more on this interpretation and other statistical properties of conformal outlier detection.

\subsection{Conformal Prediction Under Covariate Shift}
All previous conformal methods rely on Theorem~\ref{thm:conformal_calibration}, which assumes that the incoming test points come from the same distribution as the calibration points.
However, past data is not necessarily representative of future data in practice.

One type of distribution shift that conformal prediction can handle is \emph{covariate shift}.
Covariate shift refers to the situation where the distribution of $X_{\rm test}$ changes from $\mathcal{P}$ to $\mathcal{P}_{\rm test}$, but the relationship between $X_{\rm test}$ and $Y_{\rm test}$, i.e. the distribution of $Y_{\rm test}|X_{\rm test}$, stays fixed.
\begin{figure}[H]
    \centering
    \hspace{1.5cm}
    \includegraphics[width=0.34\linewidth]{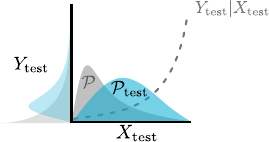}
\end{figure}
Imagine our calibration features  $\{X_i\}_{i=1}^n$ are drawn independently from $\mathcal{P}$ but our test feature $X_{\rm test}$ is drawn from $\mathcal{P}_{\rm test}$.
Then, there has been a covariate shift, and the data are no longer i.i.d.
This problem is common in the real world.
For example,
\begin{itemize}
    \item You are trying to predict diseases from MRI scans. 
    You conformalized on a balanced dataset of 50\% infants and 50\% adults, but in reality, the frequency is 5\% infants and 95\% adults.
    Deploying the model in the real world would invalidate coverage; the infants are over-represented in our sample, so diseases present during infancy will be over-predicted.
    This was a covariate shift in age.
    
    \item You are trying to do instance segmentation, i.e., to segment each object in an image from the background.
    You collected your calibration images in the morning but seek to deploy your system in the afternoon.
    The amount of sunlight has changed, and more people are eating lunch.
    This was a covariate shift in the time of day.
\end{itemize}
To address the covariate shift from $\mathcal{P}$ to $\mathcal{P}_{\rm test}$, one can form valid prediction sets with \emph{weighted conformal prediction}, first developed in~\cite{tibshirani2019conformal}.

In weighted conformal prediction, we account for covariate shift by upweighting conformal scores from calibration points that would be more likely under the new distribution.
We will be using the \emph{likelihood ratio}
\begin{equation}
    w(x) = \frac{\mathrm{d}\mathcal{P}_{\rm test}(x)}{\mathrm{d}\mathcal{P}(x)};
\end{equation}
usually this is just the ratio of the new PDF to the old PDF at the point $x$.
Now we define our weights,
\begin{equation}
    p_i^w(x) = \frac{w(X_i)}{\sum\limits_{j=1}^nw(X_j) + w(x)}\;\;\text{ and }\;\;p_{\rm test}^w(x) = \frac{w(x)}{\sum\limits_{j=1}^nw(X_j) + w(x)}.
\end{equation}
Intuitively, the weight $p_i^w(x)$ is large when $X_i$ is likely under the new distribution, and $p_{\rm test}^w(x)$ is large when the input $x$ is likely under the new distribution.
We can then express our conformal quantile as the $1-\alpha$ quantile of a reweighted distribution,
\begin{equation}
    \hat{q}(x) = \inf\left\{ s_j :  \sum\limits_{i=1}^jp_i^w(x) \ind{s_i \leq s_j} \geq 1-\alpha \right\},
\end{equation}
where above for notational convenience we assume that the scores are ordered from smallest to largest a-priori.
The choice of quantile is the key step in this algorithm, so we pause to parse it.
First of all, notice that the quantile is now a function of an input $x$, although the dependence is only minor.
Choosing $p_i^w(x)=p_{\rm test}^w(x)=\frac{1}{n+1}$ gives the familiar case of conformal prediction---all points are equally weighted, so we end up choosing the $\big\lceil (n+1)(1-\alpha) \big\rceil$th-smallest score as our quantile.
When there is covariate shift, we instead re-weight the calibration
points with non-equal weights to match the test distribution.
If the covariate shift makes easier values of $x$ more likely, it makes our quantile smaller.
This happens because the covariate shift puts more weight on small scores---see the diagram below.
Of course, the opposite holds the covariate shift upweights difficult values of $x$: so the covariate-shift-adjusted quantile grows.
\begin{figure}[H]
    \centering
    \includegraphics[width=0.4\linewidth]{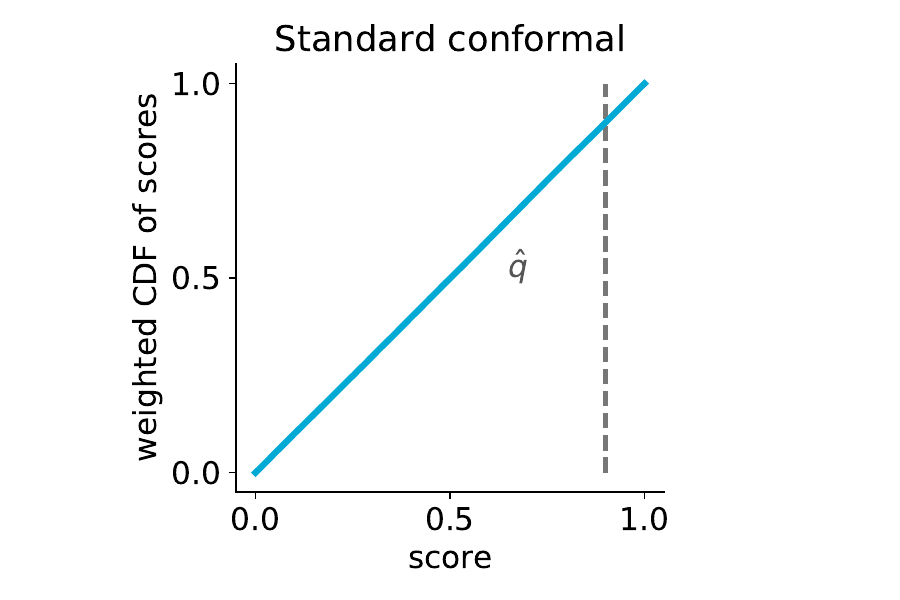}
    \includegraphics[width=0.4\linewidth]{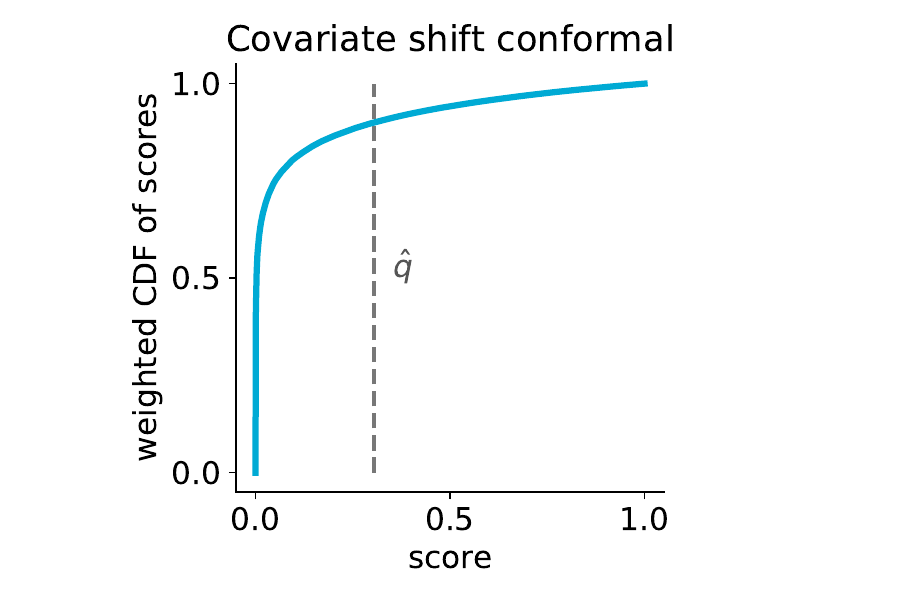}
\end{figure}

With this quantile function in hand, we form our prediction set in the standard way,
\begin{equation}
    \C(x) = \left\{ y : s(x,y) \leq \hat{q}(x)\right\}.
\end{equation}
By accounting for the covariate shift in our choice of $\hat{q}$, we were able to make our calibration data look exchangeable with the test point, achieving the following guarantee.
\begin{theorem}[Conformal prediction under covariate shift~\cite{tibshirani2019conformal}]
    Suppose $(X_1,Y_1),...,(X_n,Y_n)$ are drawn i.i.d. from $\mathcal{P} \times \mathcal{P}_{Y|X}$ and that $(X_{\rm test}, Y_{\rm test})$ is drawn independently from $\mathcal{P}_{\rm test} \times \mathcal{P}_{Y|X}$.
    Then the choice of $\C$ above satisfies
    \begin{equation}
        \P\left( Y_{\rm test} \in \C(X_{\rm test}) \right) \geq 1-\alpha.
    \end{equation}
\end{theorem}

Conformal prediction under various distribution shifts is an active and important area of research with many open challenges.
This algorithm addresses a somewhat restricted case---that of a known covariate shift---but is nonetheless quite practical.

\subsection{Conformal Prediction Under Distribution Drift}
\label{subsec:drift}
Another common form of distribution shift is \emph{distribution drift}: slowly varying changes in the data distribution.
For example, when collecting time-series data, the data distribution may change---furthermore, it may change in a way that is unknown or difficult to estimate.
Here, one can imagine using weights that give more weight to recent conformal scores.
The following theory provides some justification for such \emph{weighted conformal} procedures; in particular, they always satisfy marginal coverage, and are exact when the magnitude of the distribution shift is known. 

More formally, suppose the calibration data $\{(X_i,Y_i)\}_{i=1}^n$ are drawn independently from different distributions $\{\mathcal{P}_i\}_{i=1}^n$ and the test point $(X_{\rm test}, Y_{\rm test})$ is drawn from $\mathcal{P}_{\rm test}$.
Given some weight schedule $w_1,...,w_n$, $w_i \in [0, 1]$, we will consider the calculation of weighted quantiles using the calibration data:
\begin{equation}
  \hat{q} = \inf\left\{ q :  \sum\limits_{i=1}^n \tilde{w}_i \ind{s_i \leq q} \geq 1-\alpha\right\},
\end{equation}
where the $\tilde{w}_i$ are normalized weights,
\begin{equation}
  \tilde{w}_i = \frac{w_i}{w_1 + \ldots + w_n + 1}.
\end{equation}
Then we can construct prediction sets in the usual way,
\begin{equation}
    \C(x) = \left\{ y : s(x,y) \leq \hat{q}\right\}.
\end{equation}

We now state a theorem showing that when the distribution is shifting, it is a good idea to apply a discount factor to old samples.
In particular, let $\epsilon_i = \mathrm{d}_{\rm TV}\big( (X_i, Y_i), (X_{\rm test}, Y_{\rm test}) \big)$ be the TV distance between the $i$th data point and the test data point.
The TV distance is a measure of how much the distribution has shifted---a large $\epsilon_i$ (close to $1$) means the $i$th data point is not representative of the new test point.
The result states that if $w$ discounts those points with large shifts, the coverage remains close to $1-\alpha$. 
\begin{theorem}[Conformal prediction under distribution drift~\cite{barber2022conformal}]
\label{thm:tv-bound}
    Suppose $\epsilon_i = \mathrm{d}_{\rm TV}\big( (X_i, Y_i), (X_{\rm test}, Y_{\rm test}) \big)$.
    Then the choice of $\C$ above satisfies
    \begin{equation}
      \P\left( Y_{\rm test} \in \C(X_{\rm test}) \right) \geq 1-\alpha-2\sum\limits_{i=1}^n\tilde{w}_i\epsilon_i.
    \end{equation}
\end{theorem}
When either factor in the product $\tilde{w}_i\epsilon_i$ is small, that means that the $i$th data point doesn't result in loss of coverage. 
In other words, if there isn't much distribution shift, we can place a high weight on that data point without much penalty, and vice versa. 
Setting $\epsilon_i = 0$ above, we can also see that when there is no distribution shift, there is no loss in coverage regardless of what choice of weights is used---this fact had been observed previously in~\cite{guan2020conformal,tibshirani2019conformal}.

The $\epsilon_i$ are never known exactly in advance---we only have some heuristic sense of their size.
In practice, for time-series problems, it often suffices to pick either a rolling window of size $K$ or a smooth decay using some domain knowledge about the speed of the drift:
\begin{equation}
  w_i^{\rm fixed} = \ind{i \geq n - K} \qquad \text{ or } \qquad w_i^{\rm decay} = 0.99^{n - i + 1}.
\end{equation}
We give a worked example of this procedure for a distribution shifting over time in Section~\ref{subsec:time-series-conformal}.

As a final point on this algorithm, we note that there is some cost to using this or any other weighted conformal procedure.
In particular, the weights determine the \emph{effective sample size} of the distribution:
\begin{equation}
  n^{\rm eff}(w_1, \ldots,w_n) = \frac{w_1 + \ldots + w_n}{w_1^2 + \ldots + w_n^2}.
\end{equation}
This is quite important in practice, since the variance of the weighted conformal procedure can explode when $n^{\rm eff}$ is small; as in Section~\ref{sec:evaluating}, the variance of coverage scales as $1/\sqrt{n^{\rm eff}}$, which can be large if too many of the $w_i$ are small.
To see more of the theory of weighted conformal prediction under distribution drift, see~\cite{barber2022conformal}.

\section{Worked Examples}

We now show several worked examples of the techniques described in Section~\ref{sec:advanced}.
For each example, we provide Jupyter notebooks that allow the results to be conveniently replicated and extended.

\subsection{Multilabel Classification}
\label{subsec:multilabel-fnr}
\begin{figure}[h]
    \centering
    \includegraphics[width=\linewidth]{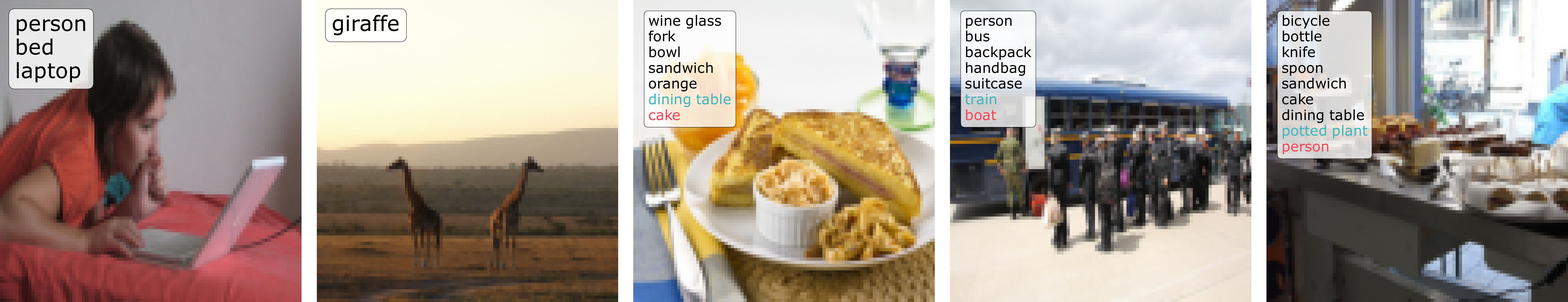}
    \caption{\textbf{Examples of false negative rate control in multilabel classification} on the MS COCO dataset with $\alpha=0.1$.
    False negatives are red, false positives are blue, and true positives are black. \jupyter{https://github.com/aangelopoulos/conformal-prediction/blob/main/notebooks/multilabel-classification-mscoco.ipynb}}
    \label{fig:coco-fnr-examples}
\end{figure}

In the multilabel classification setting, we receive an image and predict which of $K$ objects are in an image.
We have a pretrained model $\hat{f}$ that outputs estimated probabilities for each of the $K$ classes.
We wish to report on the possible classes contained in the image, returning most of the true labels. To this end, we will threshold the model's outputs to get the subset of $K$ classes that the model thinks is most likely, $\Clam(x) = \{y : \hat{f}(x) \geq \lambda\}$, which we call the prediction.
We will use conformal risk control (Section~\ref{subsec:conformal-risk}) to pick the threshold value $\lambda$ certifying a low \emph{false negative rate} (FNR), i.e., to guarantee the average fraction of ground truth classes that the model missed is less than $\alpha$.

More formally, our calibration set $\{(X_i, Y_i)\}_{i=1}^n$ contains exchangeable images $X_i$ and sets of classes $Y_i \subseteq \{1,...,K\}$.
With the notation of Section~\ref{subsec:conformal-risk}, we set our loss function to be $\ell_{\rm FNR}(\Clam(x),y)=1-|\Clam(x) \cap y|/|y|$.
Then, picking $\lhat$ as in~\ref{eq:lhat} yields a bound on the false negative rate,
\begin{equation}
    \label{eq:fnr-bound}
    \E\left[\ell_{\rm FNR}\big(\Clamhat(X_{\rm test}), Y_{\rm test}\big)\right] \leq \alpha. 
\end{equation}
Figure~\ref{fig:coco-fnr-examples} gives results and code for FNR control on the Microsoft Common Objects in Context dataset~\cite{lin2014microsoft}.

\subsection{Tumor Segmentation}
\label{subsec:tumor-fnr}
\begin{figure}[h]
    \centering
    \includegraphics[width=\linewidth]{figures/worked-examples-figures/multipolyp_grid_fig.pdf}
    \caption{\textbf{Examples of false negative rate control in tumor segmentation} with $\alpha=0.1$.
    False negatives are red, false positives are blue, and true positives are black. \jupyter{https://github.com/aangelopoulos/conformal-prediction/blob/main/notebooks/tumor-segmentation.ipynb}}
    \label{fig:tumor-fnr-examples}
\end{figure}

In the tumor segmentation setting, we receive an $M \times N \times 3$ image of a tumor and predict an $M \times N$ binary mask, where `1' indicates a tumor pixel.
We start with a pretrained segmentation model $\hat{f}$ that outputs an $M \times N$ grid of the estimated probabilities that each pixel is a tumor pixel.
We will threshold the model's outputs to get our predicted binary mask, $\Clam(x) = \{(i,j) : \hat{f}(x)_{(i,j)} \geq \lambda\}$, which we call the prediction.
We will use conformal risk control (Section~\ref{subsec:conformal-risk}) to pick the threshold value $\lambda$ certifying a low FNR, i.e., guaranteeing the average fraction of tumor pixels missed is less than $\alpha$.

More formally, our calibration set $\{(X_i, Y_i)\}_{i=1}^n$ contains exchangeable images $X_i$ and sets of tumor pixels $Y_i \subseteq \{1, \ldots, M\} \times \{1, \ldots, N\}$.
As in the previous example, we let the loss be the false negative proportion, $\ell_{\rm FNR}$. 
Then, picking $\lhat$ as in~\ref{eq:lhat} yields the bound on the FNR in~\ref{eq:fnr-bound}.
Figure~\ref{fig:tumor-fnr-examples} gives results and code on a dataset of gut polyps.

\subsection{Weather Prediction with Time-Series Distribution Shift}
\label{subsec:time-series-conformal}
\begin{figure}[h]
    \centering
\includegraphics[width=\linewidth]{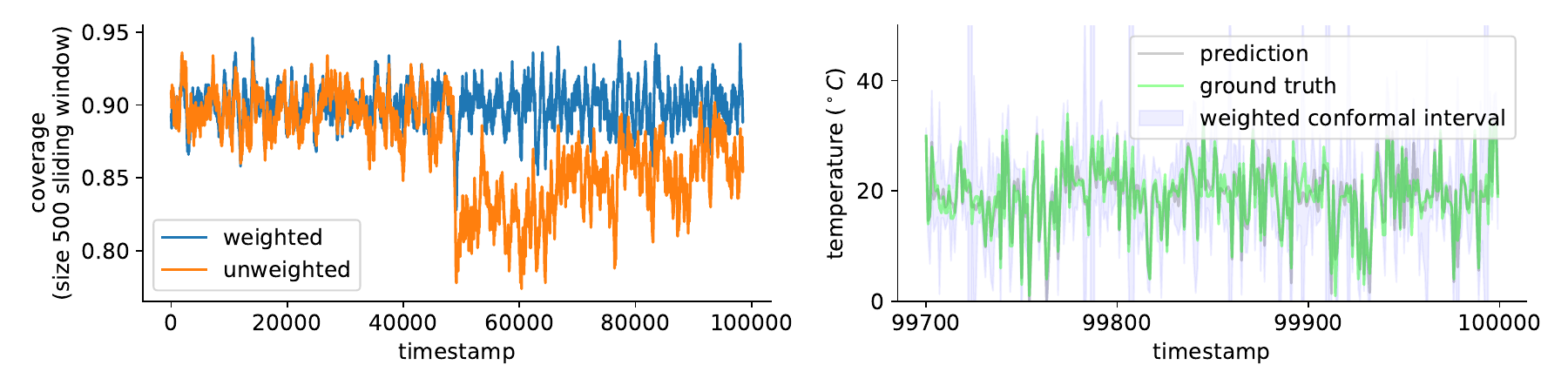}
    \caption{\textbf{Conformal prediction for time-series temperature estimation} with $\alpha=0.1$.
    On the left is a plot of coverage over time; `weighted' denotes the procedure in Section~\ref{subsec:time-series-conformal} while `unweighted' denotes the procedure that simply computes the conformal quantile on all conformal scores seen so far.
    Note that we compute coverage using a sliding window of 500 points, which explains some of the variability in the coverage.
    Running the notebook with a trailing average of 5000 points reveals that the unweighted version systematically undercovers before the change-point as well.
    On the right is a plot showing the intervals resulting from the weighted procedure.
    \jupyter{https://github.com/aangelopoulos/conformal-prediction/blob/main/notebooks/weather-time-series-distribution-shift.ipynb}}
    \label{fig:time-series-example}
\end{figure}
In this example we seek to predict the temperature of different locations on Earth given covariates such as the latitude, longitude, altitude, atmospheric pressure, and so on.
We will make these predictions serially in time.
Dependencies between adjacent data points induced by local and global weather changes violate the standard exchangeability assumption, so we will need to apply the method from Section~\ref{subsec:drift}.

In this setting, we have a time series $\big\{ (X_t,Y_t) \big\}_{t=1}^T$, where the $X_t$ are tabular covariates and the $Y_t \in \R$ are temperatures in degrees Celsius.
Note that these data points are not exchangeable or i.i.d.; adjacent data points will be correlated.
We start with a pretrained model $\hat{f}$ taking features and predicting temperature and an uncertainty model $\hat{u}$ takes features and outputs a scalar notion of uncertainty.
Following Section~\ref{subsec:uncertainty-scalar}, we compute the conformal scores
\begin{equation}
    s_t = \frac{\big| Y_t - \hat{f}(X_t) \big|}{\hat{u}(X_t)}.
\end{equation}

Since we observe the data points sequentially, we also observe the scores sequentially, and we will need to pick a different conformal quantile for each incoming data point.
More formally, consider the task of predicting the temperature at time $t \leq T$. 
We use the weighted conformal technique in Section~\ref{subsec:time-series-conformal} with the fixed $K$-sized window $w_{t'} = \ind{t' \geq t - K}$ for all $t' < t$.
This yields the quantiles
\begin{equation}
  \hat{q}_{t} = \inf\left\{ q :  \frac{1}{\min(K,t'-1)+1}\sum\limits_{t'=1}^{t-1} s_{t'} \ind{t' \geq t - K} \geq 1-\alpha\right\}.
\end{equation}
With these adjusted quantiles in hand, we form prediction sets at each time step in the usual way,
\begin{equation}
    \C(X_{t}) = \Big[ \hat{f}(X_{t}) - \hat{q}_{t}\hat{u}(X_{t}) \; , \; \hat{f}(X_{t}) + \hat{q}_{t}\hat{u}(X_{t}) \Big].
\end{equation}

We run this procedure on the Yandex Weather Prediction dataset.
This dataset is part of the Shifts Project~\cite{malinin2021shifts}, which also provides an ensemble of 10 pretrained CatBoost~\cite{dorogush2018catboost} models for making the temperature predictions.
We take the average prediction of these models as our base model $\hat{f}$.
Each of the models has its own internal variance; we take the average of these variances as our uncertainty scalar $\hat{u}$.
The dataset includes an in-distribution split of fresh data from the same time frame that the base model was trained and an out-of-distribution split consisting of time windows the model has never seen.
We concatenate these datasets in time, leading to a large change point in the score distribution.
Results in Figure~\ref{fig:time-series-example} show that the weighted method works better than a naive unweighted conformal baseline, achieving the desired coverage in steady-state and recovering quickly from the change point.
There is no hope of measuring the TV distance between adjacent data points in order to apply Theorem~\ref{thm:tv-bound}, so we cannot get a formal coverage bound.
Nonetheless, the procedure is useful with this simple fixed window of weights, which we chose with only a heuristic understanding of the distribution drift speed.
It is worth noting that conformal prediction for time-series applications is a particularly active area of research currently, and the method we have presented is not clearly the best. See~\cite{gibbs2021adaptive,zaffran2022adaptive,gibbs2022conformal} and~\cite{xu2021conformal} for two differing perspectives.

\subsection{Toxic Online Comment Identification via Outlier Detection}
\label{subsec:toxic-text}
\begin{figure}[ht]
    \centering
    \includegraphics[width=\textwidth]{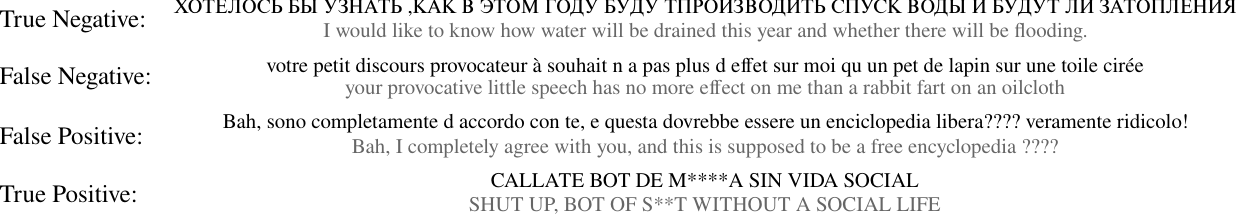}
    \caption{\textbf{Examples of toxic online comment identification with type-1 error control} at level $\alpha=0.1$ on the Jigsaw Multilingual Toxic Comment Classification dataset. \jupyter{https://github.com/aangelopoulos/conformal-prediction/blob/main/notebooks/toxic-text-outlier-detection.ipynb}}
    \label{fig:toxic-text}
\end{figure}

We provide a type-1 error guarantee on a model that flags toxic online comments, such as threats, obscenity, insults, and identity-based hate.
Suppose we are given $n$ non-toxic text samples $X_1, ..., X_n$ and asked whether a new text sample $X_{\rm test}$ is toxic.
We also have a pre-trained toxicity prediction model $\hat{f}(x) \in [0,1]$, where values closer to 1 indicate a higher level of toxicity.
The goal is to flag as many toxic comments as possible while not flagging more than $\alpha$ proportion of non-toxic comments.

The outlier detection procedure in Section~\ref{subsec:outlier} applies immediately.
First, we run the model on each calibration point, yielding conformal scores $s_i = \hat{f}(X_i)$.
Taking the toxicity threshold $\hat{q}$ to be the $\lceil (n+1)(1-\alpha) \rceil$-smallest of the $s_i$, we construct the function
\begin{equation}
    \C(x) = \begin{cases}
                \mathrm{inlier} & \hat{f}(x) \leq \hat{q} \\
                \mathrm{outlier} & \hat{f}(x) > \hat{q}.
            \end{cases}
\end{equation}  
This gives the guarantee in Proposition~\ref{prop:outlier}---no more than $\alpha$ fraction of future nontoxic text will be classified as toxic.

Figure~\ref{fig:toxic-text} shows results of this procedure using the Unitary Detoxify BERT-based model~\cite{hanu2020detoxify,devlin2018bert} on the Jigsaw Multilingual Toxic Comment Classification dataset from the WILDS benchmark~\cite{koh2021wilds}.
It is composed of comments from the talk channels of Wikipedia pages.
With a type-1 error of $\alpha=10\%$, the system correctly flags $70\%$ of all toxic comments.

\subsection{Selective Classification}
\begin{figure}[ht]
    \centering
    \includegraphics[width=\textwidth]{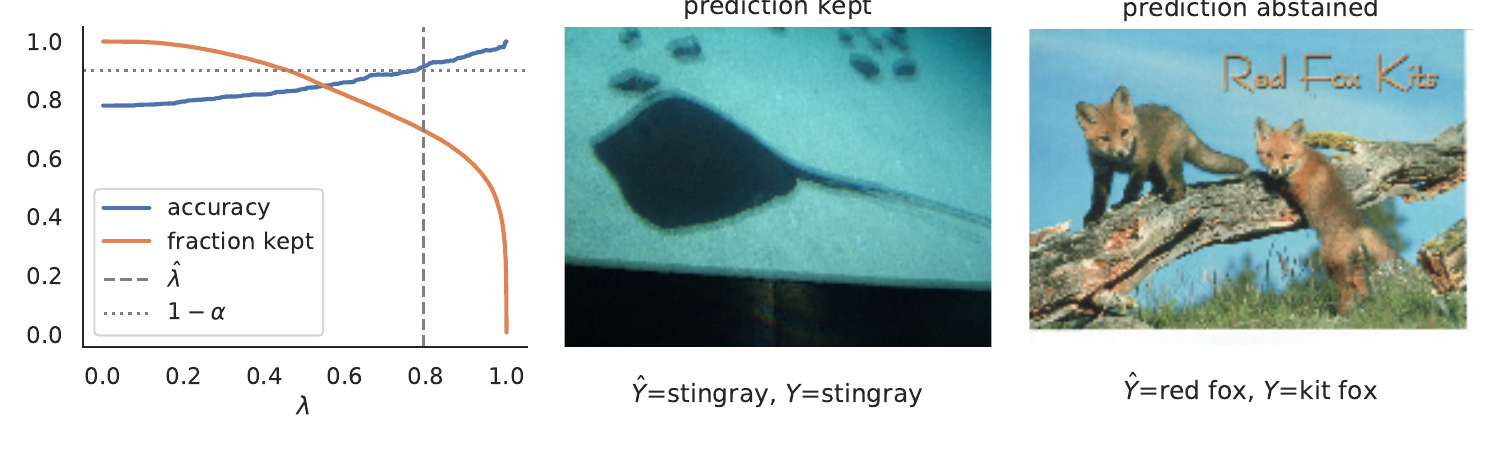}
    \caption{\textbf{Results using selective classification on Imagenet} with $\alpha=0.1$. \jupyter{https://github.com/aangelopoulos/conformal-prediction/blob/main/notebooks/imagenet-selective-classification.ipynb}}
    \label{fig:selective-classification}
\end{figure}
In many situations, we only want to show a model's predictions when it is confident.
For example, we may only want to make medical diagnoses when the model will be 95\% accurate, and otherwise to say ``I don't know."
We next demonstrate a system that strategically abstains in order to achieve a higher accuracy than the base model in the problem of image classification.

More formally, given image-class pairs $\{(X_i,Y_i)\}_{i=1}^n$ and an image classifier $\hat{f}$, we seek to ensure
\begin{equation}
    \label{eq:selective-accuracy}
    \P\left( Y_{\rm test} = \widehat{Y}(X_{\rm test}) \: \big| \widehat{P}(X_{\rm test}) \geq  \lhat \right) \geq 1-\alpha,
\end{equation}
where $\widehat{Y}(x)=\arg \max_y\,\hat{f}(x)_y$, $\widehat{P}(X_{\rm test}) = \max_y\,\hat{f}(x)_y$, and $\lhat$ is a threshold chosen using the calibration data.
This is called a \emph{selective accuracy} guarantee, because the accuracy is only computed over a subset of high-confidence predictions. This quantity cannot be controlled with techniques we've seen so far, since we are not guaranteed that model accuracy is monotone in the cutoff $\lambda$.
Nonetheless, it can be handled with Learn then Test---a framework for controlling arbitrary risks (see Appendix~\ref{app:ltt}).
We show only the special case of controlling selective classification accuracy here.

We pick the threshold using based on the empirical estimate of selective accuracy on the calibration set,
\begin{equation}
    \label{eq:rhat-selective-classification}
    \Rhat(\lambda) = \frac{1}{n(\lambda)}\sum\limits_{i=1}^n \ind{Y_{i} \neq \widehat{Y}(X_{i}) \text{ and } \widehat{P}(X_i) \geq \lambda}, \text{ where } n(\lambda) = \sum\limits_{i=1}^n \ind{ \widehat{P}(X_i) \geq \lambda }.
\end{equation}
Since this function is not monotone in $\lambda$, we will choose $\lhat$ differently than in Section~\ref{subsec:conformal-risk}. In particular, we will scan across values of $\lambda$ looking at a conservative upper bound for the true risk (i.e., the top end of a confidence interval for the selective misclassification rate).
Realizing that $\Rhat(\lambda)$ is a Binomial random variable with $n(\lambda)$ trials, we upper-bound the misclassification error as
\begin{equation}
    \Rhat^+(\lambda) = \sup\left\{ r \: : \: \text{BinomCDF}( \Rhat(\lambda) ; \: n(\lambda),r) \geq \delta \right\}
\end{equation}
for some user-specified failure rate $\delta \in [0,1]$.
Then, scan the upper bound until the last time the bound exceeds $\alpha$,
\begin{equation}
    \label{eq:lhat-selective-classification}
    \lhat = \inf\left\{\lambda : \Rhat^+(\lambda') \leq \alpha \text{ for all } \lambda' \geq \lambda \right\}.
\end{equation}
Deploying the threshold $\lhat$ will satisfy~\eqref{eq:selective-accuracy} with high probability.
\begin{prop}
Assume the $\{(X_i,Y_i)\}_{i=1}^n$ and $(X_{\rm test},Y_{\rm test})$ are i.i.d. and $\lhat$ is chosen as above. Then~\eqref{eq:selective-accuracy} is satisfied with probability $1-\delta$.
\end{prop}

See results on Imagenet at level $\alpha=0.1$ in Figure~\ref{fig:selective-classification}.
For a deeper dive into this procedure and techniques for controlling other non-monotone risks, see Appendix~\ref{app:ltt}.

\section{Full conformal prediction}
\label{sec:full_conformal}
Up to this point, we have only considered \emph{split conformal prediction}, otherwise known as inductive conformal prediction.
This version of conformal prediction is computationally attractive, since it only requires fitting the model one time, but it sacrifices statistical efficiency because it requires splitting the data into training and calibration datasets. 
Next, we consider \emph{full conformal prediction}, or transductive conformal prediction, which avoids data splitting at the cost of many more model fits. 
Historically, full conformal prediction was developed first, and then split conformal prediction was later recognized as an important special case. 
Next, we describe full conformal prediction. This discussion is motivated from three points of view. First, full conformal prediction is an elegant, historically important idea in our field. 
Second, the exposition will reveal a complimentary interpretation of conformal prediction as a hypothesis test. 
Lastly, full conformal prediction is a useful algorithm when statistical efficiency is of paramount importance.

\subsection{Full Conformal Prediction}
This topic requires expanded notation. 
Let $(X_1, Y_1), \dots, (X_{n+1}, Y_{n+1})$ be $n+1$ exchangeable data points.
As before, the user sees $(X_1, Y_1), \dots, (X_{n}, Y_{n})$ and $X_{n+1}$, and wishes to make a prediction set that contains $Y_{n+1}$. 
But unlike split conformal prediction, we allow the model to train on all the data points, so there is no separate calibration dataset.

The core idea of full conformal prediction is as follows.
We know that the true label, $Y_{n+1}$, lives somewhere in $\Y$ --- so if we loop over all possible $y \in \Y$, then we will eventually hit the data point $(X_{n+1},Y_{n+1})$, which is exchangeable with the first $n$ data points.
Full conformal prediction is so-named because it directly computes this loop.
For each $y \in \mathcal{Y}$, we fit a new model $\hat{f}^y$ to the augmented dataset $(X_1,Y_1), \ldots, (X_{n+1},y)$.
Importantly, the model fitting for $\hat{f}$ must be invariant to permutations of the data.
Then, we compute a score function $s_i^y = s(X_i, Y_i, \hat{f}^y)$ for i = 1,\dots,n and $s_{n+1}^y = s(X_{n+1}, y, \hat{f}^y)$.  
This score function is exactly the same as those from Section~\ref{sec:examples-conformal}, except that the model $\hat{f}^y$ is now given as an argument because it is no longer fixed.
Then, we calculate the conformal quantile,
\begin{equation}
    \hat{q}^y = \quantile\left( s_1^y, \ldots, s_n^y; \frac{\lceil (n+1)(1-\alpha) \rceil}{n} \right).
\end{equation}
Then, we collect all values of $y$ that are sufficiently consistent with the previous data $(X_1, Y_1), \dots, (X_{n}, Y_{n})$ are collected into a confidence set for the unknown value of $Y_{n+1}$:
\begin{equation}
    \label{eq:full-conformal-set}
    \C(X_{\rm test}) = \{ y : s^y_{n+1} \leq \hat{q}^y \}.
\end{equation}
This prediction set has the same validity guarantee as before:
\begin{theorem}[Full conformal coverage guarantee~\cite{vovk2005algorithmic}]
    \label{thm:full-conformal}
    Suppose $(X_1,Y_1),...,(X_{n+1},Y_{n+1})$ are drawn i.i.d. from $\mathcal{P}$, and that $\hat{f}$ is a symmetric algorithm.
    Then the choice of $\C$ above satisfies
    \begin{equation}
        \P\left( Y_{n + 1} \in \C(X_{n + 1}) \right) \geq 1-\alpha.
    \end{equation}
\end{theorem}
More generally, the above holds for exchangeable random variables $(X_1,Y_1),...,(X_{n+1},Y_{n+1})$; the proof of Theorem~\ref{thm:full-conformal} critically relies on the fact that the score $s_{n+1}^{Y_{n+1}}$ is exchangeable with $s_1^{Y_{n+1}}, \ldots, s_n^{Y_{n+1}}$. We defer the proof to~\cite{vovk2005algorithmic}, and note that upper bound in~\eqref{eq:coverage} also holds when the score function is continuous.

What about computation?
In principle, to compute~\eqref{eq:full-conformal-set}, we must iterate over all $y \in \mathcal{Y}$, which leads to a substantial computational burden. 
(When $\mathcal{Y}$ is continuous, we would typically first discretize the space and then check each element in a finite set.) 
For example, if $|Y|=K$, then computing~\eqref{eq:full-conformal-set} requires $(n+1) \cdot K$ model fits. 
For some specific score functions, the set in~\eqref{eq:full-conformal-set} can actually be computed exactly even for continuous $Y$, and we refer the reader to~\cite{vovk2005algorithmic} and~\cite{shafer2008tutorial} for a summary of such cases and~\cite{ndiaye2019, Ndiaye2022root-finding} for recent developments. 
Still, full conformal prediction is generally computationally costly.

Lastly, we give a statistical interpretation for the prediction set in~\eqref{eq:full-conformal-set}. 
The condition 
\begin{equation}
    s^y_{n+1} \leq \hat{q}^y
\end{equation}
is equivalent to the acceptance condition of a certain permutation test. 
To see this, consider a level $\alpha$ permutation test for the exchangeability of $s_1^y,\dots,s_n^y$ and the test score $s_{n+1}^y$, rejecting when the score function is large. 
The values of $y$ such that the test does not reject are exactly those in~\eqref{eq:full-conformal-set}.
In words, the confidence set is all values of $y$ such that the hypothetical data point is consistent with the other data, as judged by this permutation test.
We again refer the reader to~\cite{vovk2005algorithmic} for more on this viewpoint on conformal prediction.

\subsection{Cross-Conformal Prediction, CV+, and Jackknife+}

Split conformal prediction requires only one model fitting step, but sacrifices statistical efficiency. On the other hand, full conformal prediction requires a very large number of model fitting steps, but has high statistical efficiency. These are not the only two achievable points on the spectrum---there are techniques that fall in between, trading off statistical efficiency and computational efficiency differently. In particular, cross-conformal prediction~\cite{vovk2015cross} and CV+/Jackknife+~\cite{barber2021predictive} both use a small number of model fits, but still use all data for both model fitting and calibration. We refer the reader to those works for a precise description of the algorithms and corresponding statistical guarantees.

\section{Historical Notes on Conformal Prediction}
\label{sec:history}

We hope the reader has enjoyed reading the technical content in our gentle introduction.
As a d\'enouement, we now pay homage to the history of conformal prediction. Specifically, we will trace the history of techniques related to conformal prediction that are distribution-free, i.e., (1) agnostic to the model, (2) agnostic to the data distribution, and (3) valid in finite samples.
There are other lines of work in statistics with equal claim to the term ``distribution-free'' especially when it is interpreted asymptotically, such as permutation tests~\cite{chung2013exact}, quantile regression~\cite{koenker1978regression}, rank tests~\cite{mann1947test,lehmann1953power,sidak1999theory}, and even the bootstrap~\cite{efron1994introduction,chatterjee2009distribution}---the following is not a history of those topics.
Rather, we focus on the progenitors and progeny of conformal prediction.

\subsection*{Origins}
The story of conformal prediction begins sixty-three kilometers north of the seventh-largest city in Ukraine, in the mining town of Chervonohrad in the Oblast of Lviv, where Vladimir Vovk spent his childhood.
Vladimir's parents were both medical professionals, of Ukrainian descent, although the Lviv region changed hands many times over the years.
During his early education, Vovk recalls having very few exams, with grades mostly based on oral answers.
He did well in school and eventually took first place in the Mathematics Olympiad in Ukraine; he also got a Gold Medal, meaning he was one of the top graduating secondary school students.
Perhaps because he was precocious, his math teacher would occupy him in class by giving him copies of a magazine formerly edited by Isaak Kikoin and Andrey Kolmogorov, \href{https://archive.org/details/kvant-journal}{Kvant}, where he learned about physics, mathematics, and engineering---see Figure~\ref{fig:quant}.
Vladimir originally attended the Moscow Second Medical Institute (now called the Russian National Research Medical University) studying Biological Cybernetics, but eventually became disillusioned with the program, which had too much of a medical emphasis and imposed requirements to take classes like anatomy and physiology (there were ``too many bones with strange Latin names'').
Therefore, he sat the entrance exams a second time and restarted school at the Mekh-Mat (faculty of mechanics and mathematics) in Moscow State University.
In his third year there, he became the student of Andrey Kolmogorov.
This was when the seeds of conformal prediction were first laid.
Today, Vladimir Vovk is widely recognized for being the co-inventor of conformal prediction, along with collaborators Alexander Gammerman, Vladimir Vapnik, and others, whose contributions we will soon discuss.
First, we will relay some of the historical roots of conformal prediction, along with some oral history related by Vovk that may be forgotten if never written.

\begin{figure}[H]
    \centering
    \begin{minipage}{.22\textwidth}
        \centering
        \includegraphics[width=\linewidth]{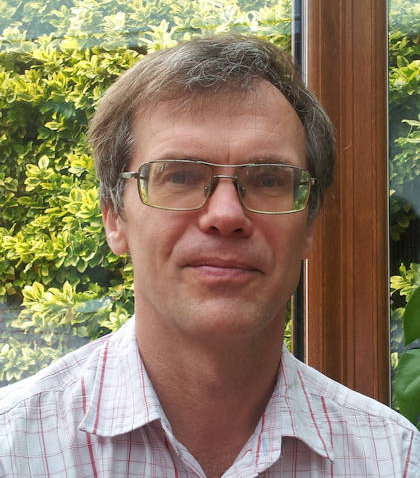}
        \caption*{\textbf{Vladimir Vovk}}
    \end{minipage}%
    \begin{minipage}{0.78\textwidth}
        \centering
        \includegraphics[width=\linewidth]{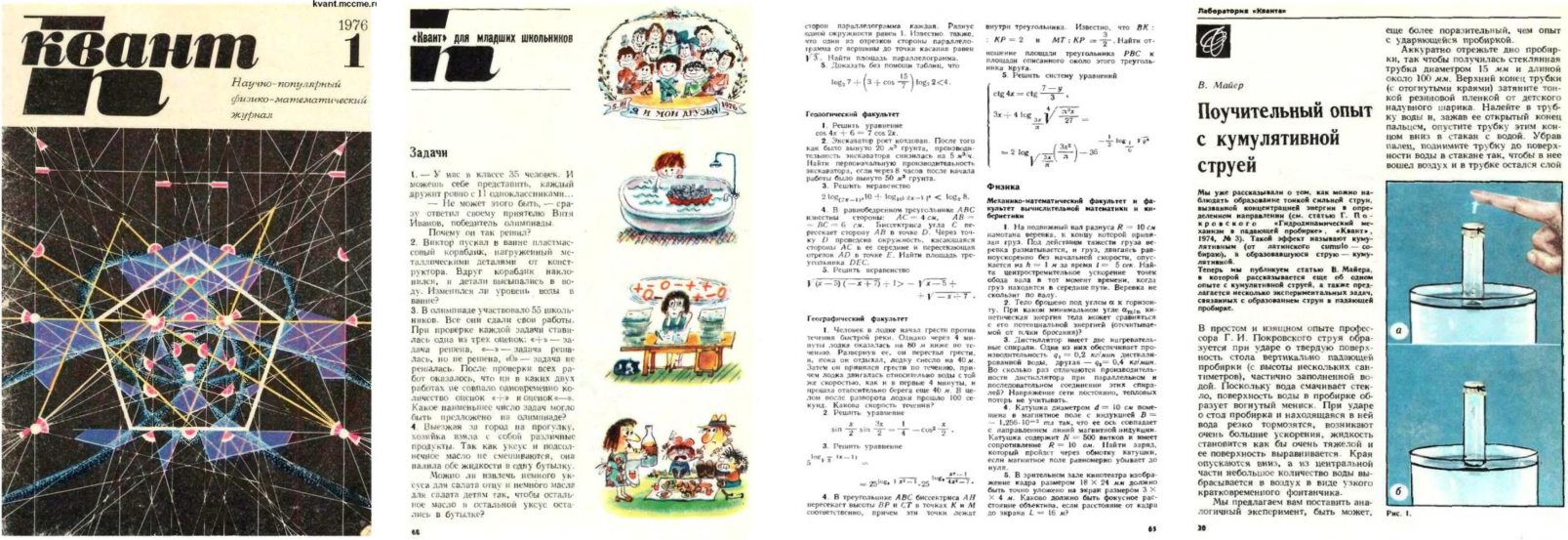}
        \caption{\textbf{Pages from the 1976 edition of Kvant magazine.}}
        \label{fig:quant}
    \end{minipage}
\end{figure}

Kolmogorov and Vovk met approximately once a week during his three remaining years as an undergraduate at MSU.
At that time, Kolmogorov took an interest in Vovk, and encouraged him to work on difficult mathematical problems.
Ultimately, Vovk settled on studying a topic of interest to Kolmogorov: algorithmically random sequences, then known as \emph{collectives}, and which were modified into \emph{Bernoulli sequences} by Kolmogorov.

Work on collectives began at the turn of the 20th century, with Gustav Fechner's \emph{Kollectivmasslehre}~\cite{fechner1897kollektivmasslehre}, and was developed significantly by von Mises~\cite{mises1919grundlagen}, Abraham Wald~\cite{wald1937widerspruchfreiheit}, Alonzo Church~\cite{church1940concept}, and so on.
A long debate ensued among these statisticians as to whether von Mises' axioms formed a valid foundation for probability, with Jean Ville being a notable opponent~\cite{ville1939etude}.
Although the theory of von Mises' collectives is somewhat defunct, the mathematical ideas generated during this time continue to have a broad impact on statistics, as we will see.
More careful historical reviews of the original debate on collectives exist elsewhere~\cite{shafer2006sources,church1940concept,vovk2001kolmogorov,porter2014kolmogorov}.
We focus on its connection to the development of conformal prediction.

Kolmogorov's interest in  \emph{Bernoulli sequences} continued into the 1970s and 1980s, when Vovk was his student.
Vovk recalls that, on the way to the train station, Kolmogorov told him (not in these exact words),
\begin{center}
    ``Look around you; you do not only see infinite sequences. There are finite sequences."
\end{center}
Feeling that the finite case was practically important, Kolmogorov extended the idea of collectives via Bernoulli sequences.
\begin{definition}[Bernoulli sequence, informal]
  A deterministic binary sequence of length n with k 1s is Bernoulli if it is a ``random" element of the set of all $\binom{n}{k}$ sequences of the same length and with the same number of 1s. ``Random" is defined as having a Kolmogorov complexity close to the maximum, $\log \binom{n}{k}$.
\end{definition}

As is typical in the study of random sequences, the underlying object itself is not a sequence of random variables. Rather, Kolmogorov quantified the ``typicality'' of a sequence via Kolmogorov complexity: he asked how long a program we would need to write in order to distinguish it from other sequences in the same space~\cite{kolmogorov1965three,kolmogorov1968logical,kolmogorov1983combinatorial}.
Vovk's first work on random sequences modified Kolmogorov's~\cite{vovk1986firstpaper} definition to better reflect the randomness in an event like a coin toss.
Vovk discusses the history of Bernoulli sequences, including the important work done by Martin-L\"of and Levin, in the Appendix of~\cite{vovk2021testing}.
Learning the theory of Bernoulli sequences brought Vovk closer to understanding finite-sample exchangeability and its role in prediction problems.

We will make a last note about the contributions of the early probabilists before moving to the modern day.
The concept of a nonconformity score came from the idea of (local) \emph{randomness deficiency}.
Consider the sequence
\begin{equation}
    00000000000000000000000000000000000000000000000000000000000000000001.
\end{equation}
With a computer, we could write a very short program to identify the `1' in the sequence, since it is atypical --- it has a \emph{large} randomness deficiency.  
But to identify any particular `0' in the sequence, we must specify its location, because it is so typical --- it has a \emph{small} randomness deficiency.
A heuristic understanding suffices here, and we defer the formal definition of randomness deficiency to~\cite{mota2013sophistication}, avoiding the notation of Turing machines and Kolmogorov complexity.
When randomness deficiency is large, a point is atypical, just like the scores we discussed in Section~\ref{sec:examples-conformal}.
These ideas, along with the existing statistical literature on tolerance intervals~\cite{wilks1941,wilks1942,wald1943,tukey1947} and works related to de Finetti's theorems on exchangeability~\cite{diaconis1980finite,aldous1985exchangeability,de1929funzione,freedman1965bernard,hewitt1955symmetric,kingman1978uses} formed the seedcorn for conformal prediction: the rough notion of collectives eventually became exchangeability, and the idea of randomness deficiency eventually became nonconformity.
Furthermore, the early literature on tolerance intervals was quite close mathematically to conformal prediction---indeed, the fact that order statistics of a uniform distribution are Beta distributed was known at the time, and this was used to form prediction regions in high probability, much like~\cite{vovk2012conditional}; more on this connection is available in Edgar Dobriban's lecture notes~\cite{dobriban2022topics}. 

\subsection*{Enter Conformal Prediction}
The framework we now call conformal prediction was hatched by Vladimir Vovk, Alexander Gammerman, Craig Saunders, and Vladimir Vapnik in the years 1996-1999, first using e-values~\cite{gammerman1998learning} and then with p-values~\cite{saunders1999transduction, vovk1999machine}.
For decades, Vovk and collaborators developed the theory and applications of conformal prediction.
Key moments include:
\begin{itemize}
    \item the 2002 proof that in online conformal prediction, the probability of error is independent across time-steps~\cite{vovk2002line};
    \item the 2002 development, along with Harris Papadopoulos and Kostas Proedrou, of split-conformal predictors~\cite{papadopoulos2002inductive};
    \item Glenn Shafer coins the term ``conformal predictor'' on December 1, 2003 while writing \emph{Algorithmic Learning in a Random World} with Vovk~\cite{vovk2005algorithmic}.
    \item the 2003 development of Venn Predictors~\cite{vovk2003self} (Vovk says this idea came to him on a bus in Germany during the Dagstuhl seminar ``Kolmogorov Complexity \& Applications'');
    \item the 2012 founding of the Symposium on Conformal and Probabilistic Prediction and its Applications (COPA), hosted in Greece by Harris Papadopoulos and colleagues;
    \item the 2012 creation of cross-conformal predictors~\cite{vovk2015cross} and Venn-Abers predictors~\cite{vovk2012venn};
    \item The 2017 invention of conformal predictive distributions~\cite{vovk2017nonparametric}.
\end{itemize}
\href{http://alrw.net/}{\emph{Algorithmic Learning in a Random World}}~\cite{vovk2005algorithmic}, by Vovk, Gammerman, and Glenn Shafer, contains further perspective on the history described above in the bibliography of Chapter 2 and the main text of Chapter 10.
Also, the book's website links to several dozen technical reports on conformal prediction and related topics.
We now help the reader understand some of these key developments.

Conformal prediction was recently popularized in the United States by the pioneering work of Jing Lei, Larry Wasserman, and colleagues~\cite{lei2011efficient,lei2014distribution,lei2013distribution,poczos2013distribution,lei2014distribution, lei2018distribution}.
Vovk himself remembers Wasserman's involvement as a landmark moment in the history of the field.
In particular, their general framework for distribution-free predictive inference in regression~\cite{lei2018distribution} has been a seminal work.
They have also, in the special cases of kernel density estimation and kernel regression, created efficient approximations to full conformal prediction~\cite{lei2013conformal,lei2014distribution}.
Jing Lei also created a fast and exact conformalization of the Lasso and elastic net procedures~\cite{lei2019fast}.
Another equally important contribution of theirs was to introduce conformal prediction to thousands of researchers, including the authors of this paper, and also Rina Barber, Emmanuel Cand\`es, Aaditya Ramdas, Ryan Tibshirani who themselves have made recent fundamental contributions.
Some of these we have already touched upon in Section~\ref{sec:examples-conformal}, such as adaptive prediction sets, conformalized quantile regression, covariate-shift conformal, and the idea of conformal prediction as indexing nested sets~\cite{gupta2020nested}.

This group also did fundamental work circumscribing the conditions under which distribution-free conditional guarantees can exist~\cite{foygel2021limits}, building on previous works by Vovk, Lei, and Wasserman that showed for an arbitrary continuous distribution, conditional coverage is impossible~\cite{vovk2012conditional,lei2014distribution,lei2018distribution}.
More fine-grained analysis of this fact has also recently been done in~\cite{lee2021distribution}, showing that vanishing-width intervals are achievable if and only if the effective support size of the distribution of $X_{\rm test}$ is smaller than the square of the sample size.

\subsection*{Current Trends}
We now discuss recent work in conformal prediction and distribution-free uncertainty quantification more generally, providing pointers to topics we did not discuss in earlier sections.
Many of the papers we cite here would be great starting points for novel research on distribution-free methods.

Many recent papers have focused on designing conformal procedures to have good practical performance according to specific desiderata like small set sizes \citep{Sadinle2016LeastAS}, coverage that is approximately balanced across regions of feature space \citep{foygel2021limits, izbicki2019flexible, romano2020classification, cauchois2020knowing, guan2020conformal, angelopoulos2020sets}, and errors balanced across classes \citep{lei2014classification, Sadinle2016LeastAS, hechtlinger2018cautious, guan2019prediction}. 
This usually involves adjusting the conformal score; we gave many examples of such adjustments in Section~\ref{sec:examples-conformal}.
Good conformal scores can also be trained with data to optimize more complicated desiderata~\cite{stutz2021learning}.
 
Many statistical extensions to conformal prediction have also emerged.
Such extensions include the ideas of risk control~\cite{angelopoulos2020sets,angelopoulos2021learn} and covariate shift~\cite{tibshirani2019conformal} that we previously discussed.
One important and continual area of work is distribution shift, where our test point has a different distribution from our calibration data.
For example, \cite{cauchois2020robust} builds a conformal procedure robust to shifts of known $f$-divergence in the score function, and adaptive conformal prediction~\cite{gibbs2021adaptive} forms prediction sets in a data stream where the distribution varies over time in an unknown fashion by constantly re-estimating the conformal quantile.
A weighted version of conformal prediction pioneered by~\cite{barber2022conformal} provides tools for addressing non-exchangeable data, most notably slowly changing time-series.
This same work develops techniques for applying full conformal prediction to asymmetric algorithms.
Beyond distribution shift, recent statistical extensions also address topics such as creating reliable conformal prediction intervals for counterfactuals and individual treatment effects~\citep{lei2020conformal,yin2021conformal,chernozhukov2021exact}, covariate-dependent lower bounds on survival times~\citep{candes2021conformalized}, prediction sets that preserve the privacy of the calibration data~\citep{angelopoulos2021private}, handling dependent data~\citep{chernozhukov2018exact, dunn2018distribution,oliveira2022split}, and achieving `multivalid' coverage that is conditionally valid with respect to several possibly overlapping groups~\citep{bastani2022practical, jung2022batch}.

Furthermore, prediction sets are not the only important form of distribution-free uncertainty quantification.
One alternative form is a \emph{conformal predictive distribution}, which outputs a probability distribution over the response space $\Y$ in a regression problem~\cite{vovk2017nonparametric}.
Recent work also addresses the issue of calibrating a scalar notion of uncertainty to have probabilistic meaning via histogram binning~\cite{gupta2021distribution, park2021pac}---this is like a rigorous version of Platt scaling or isotonic regression.
The tools from conformal prediction can also be used to identify times when the distribution of data has changed by examining the score function's behavior on new data points.
For example, \cite{bates2021multiple} performs outlier detection using conformal prediction, \cite{vovk2021testing,volkhonskiy2017inductive} detect change points in time-series data, \cite{hu2020distributionfree} tests for covariate shift between two datasets, and \cite{podkopaev2021tracking} tracks the risk of a predictor on a data-stream to identify when harmful changes in its distribution (one that increases the risk) occur.

Developing better estimators of uncertainty improves the practical effectiveness of conformal prediction.
The literature on this topic is too wide to even begin discussing; instead, we point to quantile regression as an example of a fruitful line of work that mingled especially nicely with conformal prediction in Section~\ref{subsec:cqr}.
Quantile regression was first proposed in~\cite{koenker1978regression} and extended to the locally polynomial case in~\cite{chaudhuri1991global}.
Under sufficient regularity, quantile regression converges uniformly to the true quantile function~\cite{chaudhuri1991global,steinwart2011estimating,takeuchi2006nonparametric, zhou1996direct,zhou1998statistical}.
Practical and accessible references for quantile regression have been written by Koenker and collaborators~\cite{koenker2005quantile,koenker2018handbook}.
Active work continues today to analyze the statistical properties of quantile regression and its variants under different conditions, for example in additive models~\cite{koenker2011additive} or to improve conditional coverage when the size of the intervals may correlate with miscoverage events~\cite{feldman2021improving}.
The Handbook of Quantile Regression~\cite{koenker2018handbook} includes more detail on such topics, and a memoir of quantile regression for the interested reader.
Since quantile regression provides intervals with near-conditional coverage asymptotically, the conformalized version inherits this good behavior as well.

Along with such statistical advances has come a recent wave of practical applications of conformal prediction.
Conformal prediction in large-scale deep learning was studied in~\cite{angelopoulos2020sets}, focusing on image classification.
One compelling use-case of conformal prediction is speeding up and decreasing the computational cost of the test-time evaluation of complex models~\citep{fisch2020efficient,schuster2021consistent}.
The same researchers pooled information across multiple tasks in a meta-learning setup to form tight prediction sets for few-shot prediction~\citep{fisch2021few}.
There is also an earlier line of work, appearing slightly after that of Lei and Wasserman, applying conformal prediction to decision trees~\cite{johansson2014regression, linusson2017calibration, bostrom2017accelerating}.
Closer to end-users, we are aware of several real applications of conformal prediction.
The Washington Post estimated the number of outstanding Democratic and Republican votes in the 2020 United States presidential election using conformal prediction~\cite{cherian2020washington}.
Early clinical experiments in hospitals underscore the utility of conformal prediction in that setting as well, although real deployments are still to come~\cite{lu2021distribution,lu2021fair}.
Fairness and reliability of algorithmic risk forecasts in the criminal justice system improves (on controlled datasets) when applying conformal prediction~\cite{Romano2020With, kuchibhotla2021nested,lu2021fair}.
Conformal prediction was recently applied to create safe robotic planning algorithms that avoid bumping into objects~\cite{lindemann2022safe,dixit2022adaptive}.
Recently a \texttt{scikit-learn} compatible open-source library, \href{https://github.com/scikit-learn-contrib/MAPIE}{\texttt{MAPIE}}, has been developed for constructing conformal prediction intervals.
There remains a mountain of future work in these applications of conformal prediction and many others.

Today, the field of distribution-free uncertainty quantification remains small, but grows rapidly year-on-year.
The promulgation of machine learning deployments has caused a reckoning that point predictions are not enough and shown that we still need rigorous statistical inference for reliable decision-making.
Many researchers around the world have keyed into this fact and have created new algorithms and software using distribution-free ideas like conformal prediction.
These developments are numerous and high-quality, so most reviews are out-of-date. 
To keep track of what gets released, the reader may want to see the \href{https://github.com/valeman/awesome-conformal-prediction}{\texttt{Awesome Conformal Prediction}} repository~\cite{acp}, which provides a frequently-updated list of resources in this area.

We will end our Gentle Introduction with a personal note to the reader---you can be part of this story too.
The infant field of distribution-free uncertainty quantification has ample room for significant technical contributions.
Furthermore, the concepts are practical and approachable; they can easily be understood and implemented in code.
Thus, we encourage the reader to try their hand at distribution-free uncertainty quantification; there is a lot more to be done!

\printbibliography

\clearpage
\appendix

\section{Distribution-Free Control of General Risks}
\label{app:ltt}

\begin{figure}[t]
    \vspace{-0.3cm}
    \centering
    \includegraphics[width=\textwidth]{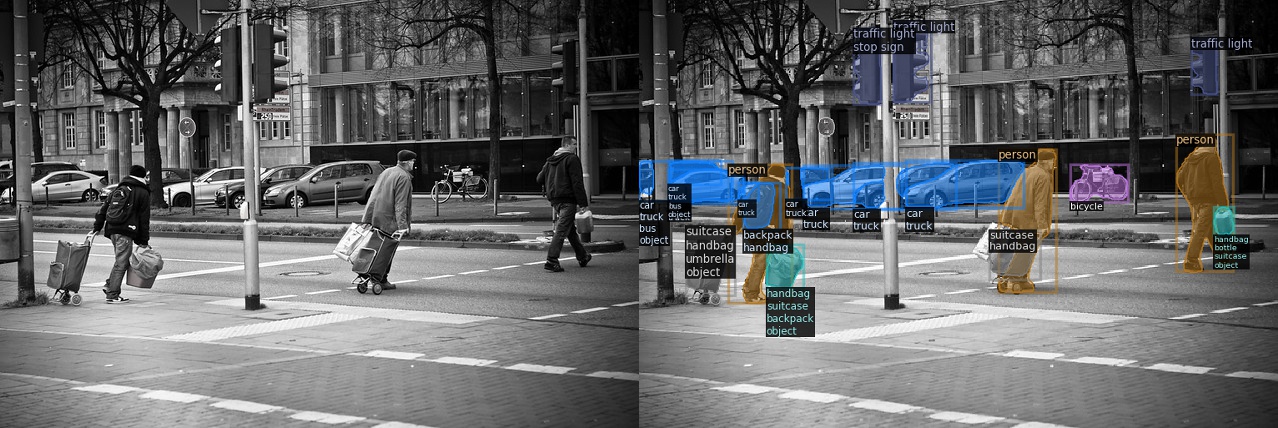}
    \caption{\textbf{Object detection with simultaneous distribution-free guarantees} on the expected intersection-over-union, recall, and coverage rate.}
    \label{fig:ltt-teaser}
\end{figure}

For many prediction tasks, the relevant notion of reliability is not coverage.
Indeed, many applications have problem-specific performance metrics---from false-discovery rate to fairness---that directly encode the soundness of a prediction.
In Section~\ref{subsec:conformal-risk}, we saw how to control the expectation of monotone loss functions using conformal risk control.
Here, we generalize further to control \emph{any} risk and multiple risks in a distribution-free way without retraining the model.
As an example, in instance segmentation, we are given an image and asked to identify all objects in the image, segment them, and classify them.
All three of these sub-tasks have their own risks: recall, \emph{intersection-over-union} (IOU), and coverage respectively.
These risks can be  automatically controlled using distribution-free statistics, as we preview in Figure~\ref{fig:ltt-teaser}.

We first re-introduce the theory of risk control below, then give a list of illustrative examples.
As in conformal risk control, we start with a pretrained model $\hat{f}$.
The model also has a \emph{parameter} $\lambda$, which we are free to choose.
We use $\hat{f}(x)$ and $\lambda$ to form our prediction, $\T_{\lambda}(x)$, which may be a set or some other object.
For example, when performing regression, $\lambda$ could threshold the estimated probability density, as below.
\begin{figure}[H]
    \centering
    \includegraphics[width=0.34\textwidth]{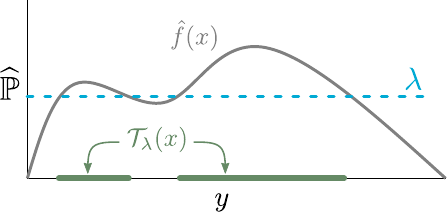}
\end{figure}

We then define a notion of risk $R(\lambda)$.
The risk function measures the quality of $\T_{\lambda}$ according to the user.
The goal of risk control is to use our calibration set to pick a parameter $\lhat$ so that the risk is small with high probability.
In formal terms, for a user-defined \emph{risk tolerance} $\alpha$ and \emph{error rate} $\delta$, we seek to ensure
\begin{equation}
    \label{eq:risk-control}
    \P\left( R\big(\hat{\lambda}\big) < \alpha \right) \geq 1-\delta,
\end{equation}
where the probability is taken over the calibration data used to pick $\lhat$.
Note that this guarantee is high-probability, unlike that in Section~\ref{subsec:conformal-risk}, which is in expectation.
We will soon introduce a distribution-free technique called \emph{Learn then Test} (LTT) for finding $\lhat$ that satisfy~\eqref{eq:risk-control}.
Below we include two example applications of risk control which would be impossible with conformal prediction and conformal risk control.
\begin{itemize}
    \item {\em Multi-label Classification with FDR Control:}
    In this setting, $X_{\rm test}$ is an image and $Y_{\rm test}$ is a subset of $K$ classes contained in the image.
    Our model $\hat{f}$ gives us the probability each of the $K$ classes is contained in the image.
    We will include a class in our estimate of $y$ if $\hat{f}_k > \lambda$ --- i.e., the parameter $\lambda$ thresholds the estimated probabilities.
    We seek to find the $\lhat$s that guarantees our predicted set of labels is sufficiently reliable as measured by the \emph{false-discovery rate} (FDR) risk $R(\lhat)$.
    
    \item {\em Simultaneous Guarantees on OOD Detection and Coverage:} For each input $X_{\rm test}$ with true class $Y_{\rm test}$, we want to decide if it is out-of-distribution.
    If so, we will flag it as such.
    Otherwise, we want to output a prediction set that contains the true class with 90\% probability.
    In this case, we have two models: $\mathrm{OOD}(x)$, which tells us how OOD the input is, and $\hat{f}(x)$, which gives the estimated probability that the input comes from each of $K$ classes.
    In this case, $\lambda$ has two coordinates, and we also have two risks.
    The first coordinate $\lambda_1$ tells us where to threshold $\mathrm{OOD}(x)$ such that the fraction of false alarms $R_1$ is controlled.
    The second coordinate $\lambda_2$ tells us how many classes to include in the prediction set to control the miscoverage $R_2$ among points identified as in-distribution.
    We will find $\lhat$s that control both $R_1(\lhat)$ and $R_2(\lhat)$ jointly.
\end{itemize}

We will describe each of these examples in detail in Section~\ref{app:ltt-examples}.
Many more worked examples, including the object detection example in Figure~\ref{fig:ltt-teaser}, are available in the cited literature on risk control~\cite{bates2021distribution,angelopoulos2021learn}.
First, however, we will introduce the general method of risk control via Learn then Test.

\subsection{Instructions for Learn then Test}
\label{app:ltt-methods}
First, we will describe the formal setting of risk control. 
We introduce notation and the risk-control property in Definition~\ref{def:rcp}. Then, we describe the calibration algorithm.

\subsection*{Formal notation for error control}

Let $(X_i, Y_i)_{i = 1,\dots,n}$ be an independent and identically distributed (i.i.d.) set of variables, where the features $X_i$ take values in $\X$ and the responses $Y_i$ take values in $\Y$.  
The researcher starts with a pre-trained predictive model $\hat{f}$.
We show how to subsequently create predictors from $\hat{f}$ that control a risk, regardless of the quality of the initial model fit or the distribution of the data. 

Next, let $\T_\lambda : \X \to \Y'$ be a function with parameter $\lambda$ that maps a feature to a prediction ($\Y'$ can be any space, including the space of responses $\Y$ or prediction sets $2^{\Y}$). 
This function $\T_{\lambda}$ would typically be constructed from the predictive model, $\hat{f}$, as in our earlier regression example.
We further assume $\lambda$ takes values in a (possibly multidimensional) discrete set $\Lambda$.
If $\Lambda$ is not naturally discrete, we usually discretize it finely.
For example, $\Lambda$ could be the set $\{0, 0.001, 0.002, ..., 0.999, 1\}$.

We then allow the user to choose a {\em risk} for the predictor $\T_{\lambda}$.
This risk can be any function of $\T_{\lambda}$, but often we take the risk function to be the expected value of a \emph{loss function},
\begin{equation}
    \label{eq:risk-expectation-loss}
    R(\T_{\lambda}) =  \E\left[\underbrace{L\big(\T_{\lambda}(X_{\rm test}),Y_{\rm test}\big)}_{\text{Loss function}}\right].
\end{equation}
The loss function is a deterministic function that is high when $\T_{\lambda}(X_{\rm test})$ does badly at predicting $Y_{\rm test}$.
The risk then averages this loss over the distribution of $(X_{\rm test},Y_{\rm test})$.
For example, taking 
\begin{equation*}
R_{\text{miscoverage}}\big(\T_{\lambda}) = \E\big[\ind{ Y_{\rm test} \notin \T_{\lambda}(X_{\rm test})}\big] = \P\left(Y_{\rm test} \notin \T_{\lambda}(X_{\rm test})\right)
\end{equation*}
gives us the familiar case of controlling miscoverage.

To aid the reader, we point out some facts about~\eqref{eq:risk-expectation-loss} that may not be obvious.
The input $\T_{\lambda}$ to the risk is a function; this makes the risk a \emph{functional} (a function of a function).
When we plug $\T_{\lambda}$ into the risk, we take an expectation of the loss over the randomness in a single test point.
At the end of the process, for a deterministic $\lambda$, we get a deterministic scalar $R(\T_{\lambda})$.
Henceforth, for ease of notation, we abbreviate this number as $R(\lambda) := R(\T_{\lambda})$.

Our goal is control the risk in the following sense:
\begin{definition}[Risk control]
Let $\hat{\lambda}$ be a random variable taking values in $\Lambda$ (i.e., the output of an algorithm run on the calibration data).
We say that $\T_{\lhat}$ is a \emph{$(\alpha,\delta)$-risk-controlling prediction} (RCP) if, with probability at least $1-\delta$, we have $R\big(\lhat\big) \le \alpha$.
\label{def:rcp}
\end{definition}
In Definition~\ref{def:rcp}, we plug in a \emph{random parameter $\lhat$} which is chosen based on our calibration data; therefore, $R(\lhat)$ is random even though the risk is a deterministic function. 
The high-probability portion of Definition~\ref{def:rcp} therefore says that $\lhat$ can only violate risk control if we choose a bad calibration set; this happens with probability at most $\delta$.
The distribution of the risk over many resamplings of the calibration data should therefore look as below.
\begin{figure}[H]
    \centering
    \includegraphics[width=0.4\linewidth]{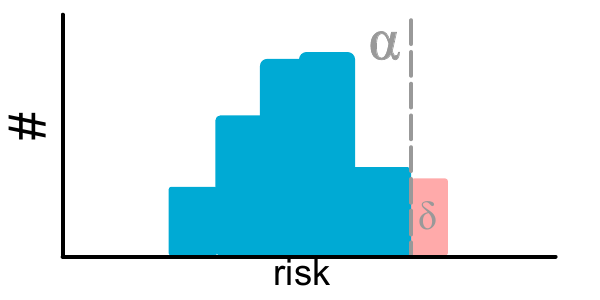}
\end{figure}
\newpage
 
\subsection*{The Learn then Test procedure}
Recalling Definition~\ref{def:rcp}, our goal is to find a set function whose risk is less than some user-specified threshold $\alpha$. 
To do this, we search across the collection of functions $\{\T_\lambda\}_{\lambda \in \Lambda}$ and estimate their risk on the calibration data $(X_1,Y_1),\dots,(X_n,Y_n)$.
The output of the procedure will be a set of $\lambda$ values which are all guaranteed to control the risk, $\Lhat \subseteq \Lambda$.
The Learn then Test procedure is outlined below.

\begin{enumerate}

    \item For each $\lambda \in \Lambda$, associate the null hypothesis $\Hlam : R(\lambda) > \alpha$. 
    Notice that \emph{rejecting} the $\Hlam$ means you selected $\lambda$ as a point where the risk is controlled. Here we denote each null with a blue dot; the yellow dot is highlighted, so we can keep track of it as we explain the procedure.
    
    \begin{figure}[H]
        \vspace{-0.4cm}
        \centering
        \includegraphics[width=0.65\linewidth]{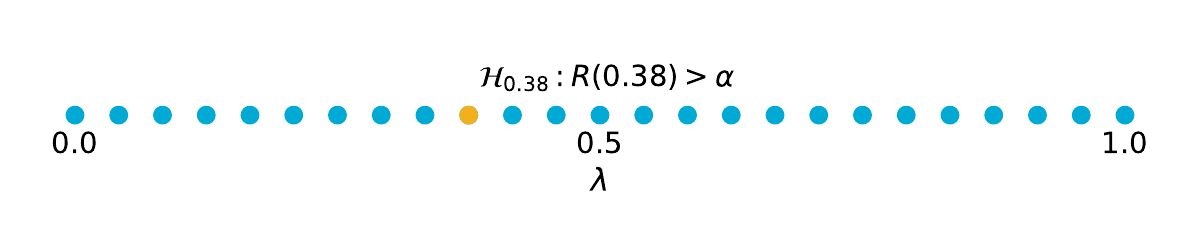}
        \vspace{-0.5cm}
    \end{figure}
    
    \item For each null hypothesis, compute a p-value using a concentration inequality. For example, Hoeffding's inequality yields $p_{\lambda}=e^{-2n(\alpha-\Rhat(\lambda))_+^2}$, where $\Rhat(\lambda)=\frac{1}{n}\sum\limits_{i=1}^nL(\T_{\lambda}(X_i),Y_i)$. 
    We remind the reader what a p-value is, why it is relevant to risk control, and point to references with stronger p-values in~\ref{sec:p-values}.
    
    \begin{figure}[H]
        \centering
        \includegraphics[width=0.48\linewidth]{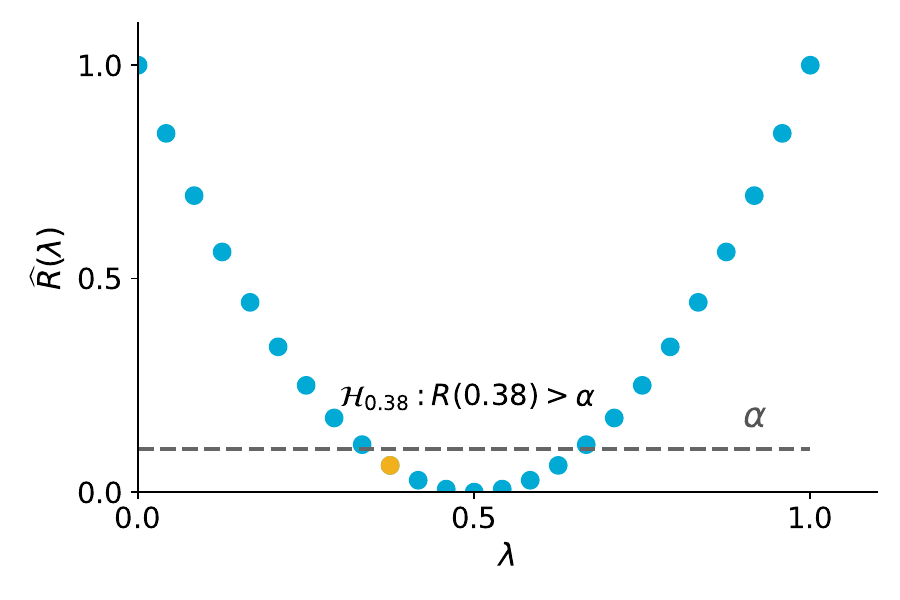}
        \includegraphics[width=0.48\linewidth]{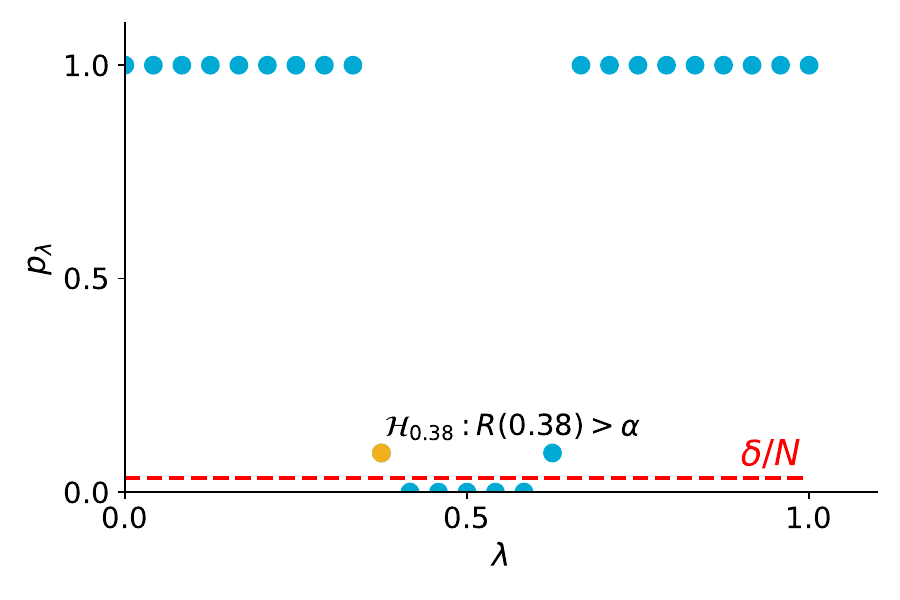}
    \end{figure}
    
    \item Return $\Lhat = \mathcal{A}\big(\{p_{\lambda}\}_{\lambda \in \Lambda}\big)$, where $\mathcal{A}$ is an algorithm that controls the familywise-error rate (FWER). 
    For example, the Bonferroni correction yields $\Lhat=\big\{\lambda : p_{\lambda} < \frac{\delta}{|\Lambda|}\big\}$. 
    We define the FWER and preview ways to design good FWER-controlling procedures in Section~\ref{sec:fwer-control}.
    The nulls with red crosses through them below have been rejected by the procedure; i.e., they all control the risk with high probability.
    
    \begin{figure}[H]
        \vspace{-0.4cm}
        \centering
        \includegraphics[width=0.65\linewidth]{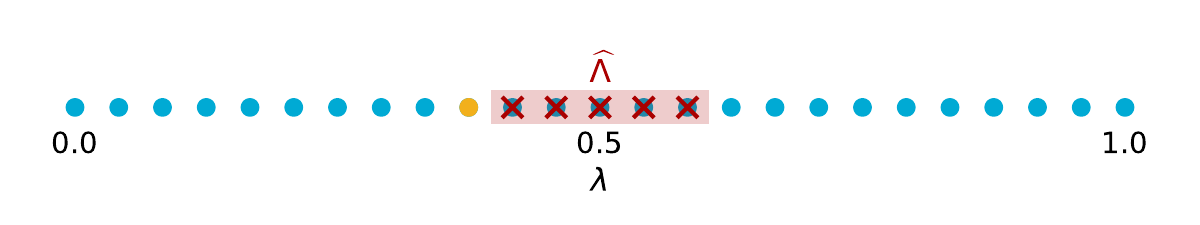}
        \vspace{-0.5cm}
    \end{figure}
    
\end{enumerate}

By following the above procedure, we get the statistical guarantee in Theorem~\ref{thm:ltt}.
\begin{theorem}
    \label{thm:ltt}
    The $\Lhat$ returned by the Learn then Test procedure satisfies
    \begin{equation}
    \label{eq:fwer_control}
        \P \left(\sup_{\lhat\in \Lhat}\{R(\lhat)\} \le \alpha \right) \geq 1-\delta.
    \end{equation}
    Thus, selecting any $\lhat \in \Lhat$, $\T_{\hat{\lambda}}$ is an $(\alpha,\delta)$-RCP.
\end{theorem}

The LTT procedure decomposes risk control into two subproblems: computing p-values and combining them with multiple testing.
We will now take a closer look at each of these subproblems.

\begin{figure}
    \centering
    \begin{minted}[fontsize=\footnotesize]{python}
# Implementation of LTT. Assume access to X, Y where n=X.shape[0]=Y.shape[0]
lambdas = torch.linspace(0,1,N) # Commonly choose N=1000
losses = torch.zeros((n,N)) # Compute the loss function next
for (i,j) in [(i,j) for i in range(n) for j in range(N)]:
  prediction_set = T(X[i],lambdas[j]) # T ( ) is problem depemdent
  losses[i,j] = get_loss(prediction_set,Y[i]) # Loss is problem dependent
risk = losses.mean(dim=0)
pvals = torch.exp(-2*n*(torch.relu(alpha-risk)**2)) # Or any p-value
lambda_hat = lambdas[pvals<delta/lambdas.shape[0]] # Or any FWER-controlling algorithm 
    \end{minted}
    \caption{\textbf{PyTorch code for running Learn then Test.}}
    \label{fig:ltt-code}
\end{figure}
\subsubsection{Crash Course on Generating p-values}
\label{sec:p-values}
\textbf{What is a p-value, and why is it related to risk control?} In Step 1 of the LTT procedure, we associated a null hypothesis $\Hlam$ to every $\lambda \in \Lambda$.
When the null hypothesis at $\lambda$ holds, the risk is \emph{not} controlled for that value of the parameter.
In this reframing, our task is to automatically identify points $\lambda$ where the null hypothesis does not hold---i.e., to \emph{reject the null hypotheses} for some subset of $\lambda$ such that $R(\lambda) \leq \alpha$.
The process of accepting or rejecting a null hypothesis is called \emph{hypothesis testing}.

\begin{align}
    \text{Rejecting the null hypothesis $\Hlam$ } &\to \text{ the risk \emph{is} controlled at $\lambda$.} \\
    \text{Accepting the null hypothesis $\Hlam$ } &\to \text{ the risk \emph{is not} controlled at $\lambda$.}
\end{align}

In order to reject a null hypothesis, we need to have empirical evidence that at $\lambda$, the risk is controlled.
We use our calibration data to summarize this information in the form of a \emph{p-value} $p_{\lambda}$.
A p-value must satisfy the following condition, which we sometimes refer to as \emph{validity} or \emph{super-uniformity},
\begin{equation}
    \label{eq:superuniformity}
    \forall t \in [0,1], \; \; \P_{\Hlam}\left( p_{\lambda} \leq t\right) \leq t,
\end{equation}
where $\P_{\Hlam}$ refers to the probability under the null hypothesis.
Parsing the super-uniformity condition carefully tells us that when $p_{\lambda}$ is low, there is evidence against the null hypothesis $\Hlam$.
In other words, for a particular $\lambda$, we can reject $\Hlam$ if $p_{\lambda}<5\%$ and expect to be wrong no more than $5\%$ of the time.
This process is called \emph{testing the hypothesis at level $\delta$}, where in the previous sentence, $\delta=5\%$.
\begin{figure}[H]
    \centering
    \includegraphics[width=0.48\linewidth]{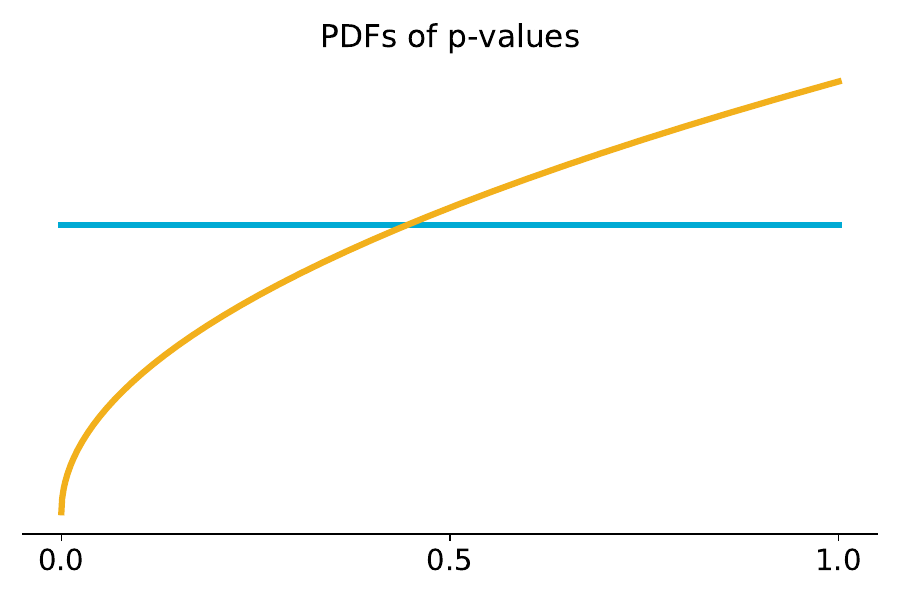}
    \includegraphics[width=0.48\linewidth]{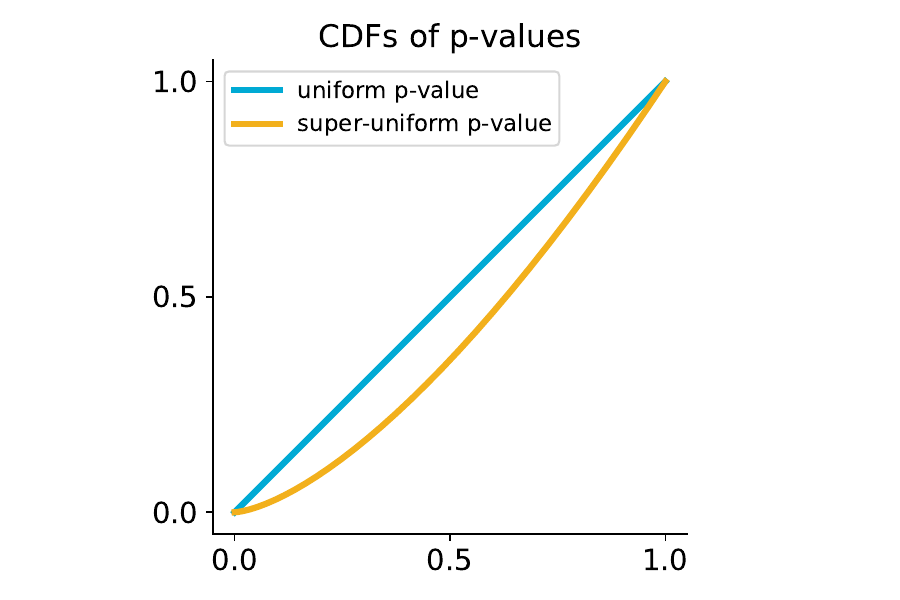}
\end{figure}

One of the key ingredients in Learn then Test is a p-value with distribution-free validity: it is valid under without assumptions on the data distribution.
For example, when working with risk functions that take values in $[0,1]$---like coverage, IOU, FDR, and so on---the easiest choice of p-value is based on Hoeffding's inequality:
\begin{equation}
    \label{eq:hoeffding-p-value}
    p_{\lambda}^{\rm Hoeffding}=e^{-2n\big(\alpha-\Rhat(\lambda)\big)_+^2}.
\end{equation}
More powerful p-values based on tighter concentration bounds are included in~\cite{angelopoulos2021learn}.
In particular, many of the practical examples in that reference use a stronger p-value called the \textit{Hoeffding-Bentkus} (HB) p-value,
\begin{align}
    \label{eq:HB-p-value}
    &p_{\lambda}^{\rm HB} = \min\left( \exp\{-nh_1(\Rhat(\lambda) \wedge \alpha , \alpha)\}, e\P\big(\mathrm{Bin}(n,\alpha) \leq \left\lceil n\Rhat(\lambda) \right\rceil \big) \right),\\
    &\text{where } h_1(a,b) = a\log\left(\frac{a}{b}\right)+(1-a)\log\left(\frac{1-a}{1-b}\right).
\end{align}
Note that any valid p-value will work---it is fine for the reader to keep $p_{\lambda}^{\rm Hoeffding}$ in mind for the rest of this manuscript, with the understanding that more powerful choices are available.

\subsubsection{Crash Course on Familywise-Error Rate Algorithms}
\label{sec:fwer-control}

If we only had one hypothesis $H_{\lambda}$, we could simply test it at level $\delta$.
However, we have one hypothesis for each $\lambda \in \Lambda$, where $|\Lambda|$ is often very large (in the millions or more).
This causes a problem: the more hypotheses we test, the higher chance we incorrectly reject at least one hypothesis.
We can formally reason about this with the \emph{familywise-error rate} (FWER).
\begin{definition}[familywise-error rate]
  The familywise-error rate of a procedure returning $\hat{\Lambda}$ is the probability of making at least one false rejection, i.e., 
  \begin{equation}
      \mathrm{FWER}\left(\Lhat\right) = \P\left( \exists \lhat \in \Lhat : R(\lhat) > \alpha \right).
  \end{equation}
\end{definition}
As a simple example to show how naively thresholding the p-values at level $\delta$ fails to control FWER, consider the case where all the hypotheses are null, and we have uniform p-values independently tested at level $\delta$. The FWER then approaches $1$; see below.
\begin{equation}
    \text{If we take } \Lhat = \{\lambda : p_{\lambda} < \delta\} \text{, then } \mathrm{FWER}(\Lhat) = 1-(1-\delta)^{|\Lambda|}.
\end{equation}
This simple toy analysis exposes a deeper problem: without an intelligent strategy for combining the information from many p-values together, we can end up making false rejections with high probability.
Our challenge is to intelligently combine the p-values to avoid this issue of multiplicity (without assuming the p-values are independent).

This fundamental statistical challenge has led to a decades-long and continually rich area of research called \emph{multiple hypothesis testing}.
In particular, a genre of algorithms called \emph{FWER-controlling algorithms} seek to select the largest set of $\Lhat$ that guarantees $\mathrm{FWER}(\Lhat) \leq \delta$.
The simplest FWER-controlling algorithm is the \emph{Bonferroni correction},
\begin{equation}
    \Lhat_{\rm Bonferroni} = \left\{ \lambda \in \Lambda : p_{\lambda} \leq \frac{\delta}{|\Lambda|}\right\}.
\end{equation}

Under the hood, the Bonferroni correction simply tests each hypothesis at level $\delta/|\Lambda|$, so the probability there exists a failed test is no more than $\delta$ by a union bound.
It should not be surprising that there exist improvements on Bonferroni correction.

First, we will discuss one important improvement in the case of a monotone loss function: \emph{fixed-sequence testing}.
As the name suggests, in fixed-sequence testing, we construct a sequence of hypotheses $\{\mathcal{H}_{\lambda_j}\}_{j=1}^{N}$ where N = $|\Lambda|$, before looking at our calibration data.
Usually, we just sort our hypotheses from most- to least-promising based on information we knew a-priori.
For example, if large values of $\lambda$ are more likely to control the risk, $\{\lambda_j\}_{j=1}^N$ just sorts $\Lambda$ from greatest to least.
Then, we test the hypotheses sequentially in some fixed order at level $\delta$, including them in $\Lhat$ as we go, and stopping when we make our first acceptance:

\begin{equation}
    \label{eq:fixed-sequence-lhat}
    \Lhat_{\rm FST} = \{\lambda_j, j\leq T\}\text{, where } T = \max\left\{ t \in \{1,...,N\} : p_{\lambda_{t'}} \leq \delta \text{, for all } t' \leq t \right\}.
\end{equation}

\begin{figure}[H]
    \centering
    \includegraphics[width=.4\linewidth]{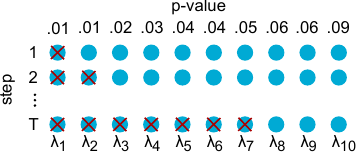}
    \caption{\textbf{An example of fixed-sequence testing} with $\delta=0.05$. Each blue circle represents a null, and each row a step of the procedure. The nulls with a red cross have been rejected at that step.}
    \vspace{-0.5cm}
\end{figure}
This sequential procedure, despite testing all hypotheses it encounters at level $\delta$, still controls the FWER.
For monotone and near-monotone risks, such as the false-discovery rate, it works quite well.

It is also possible to extend the basic idea of fixed-sequence testing to non-monotone functions, creating powerful and flexible FWER-controlling procedures using an idea called sequential graphical testing~\cite{bretz2009graphical}.
Good graphical FWER-controlling procedures can be designed to have high power for particular problems, or alternatively, automatically discovered using data. 
This topic is given a detailed treatment  in~\cite{angelopoulos2021learn}, and we omit it here for simplicity. 

We have described a general-purpose pipeline for distribution-free risk control.
It is described in PyTorch code in Figure~\ref{fig:ltt-code}.
Once the user sets up the problem (i.e., picks $\Lambda$, $\T_{\lambda}$, and $R$), the LTT pipeline we described above automatically produces $\Lhat$.
We now go through three worked examples which teach the reader how to choose $\Lambda$, $\T$ and $R$ in practical circumstances.

\section{Examples of Distribution-Free Risk Control}
\label{app:ltt-examples}

In this section, we will walk through several examples of distribution-free risk control applied to practical machine learning problems.
The goal is again to arm the reader with an arsenal of pragmatic prototypes of distribution-free risk control that work on real problems.

\subsection{Multi-label Classification with FDR Control}
\label{app:ltt-multilabel}

\begin{figure}
    \centering
    \includegraphics[width=\linewidth]{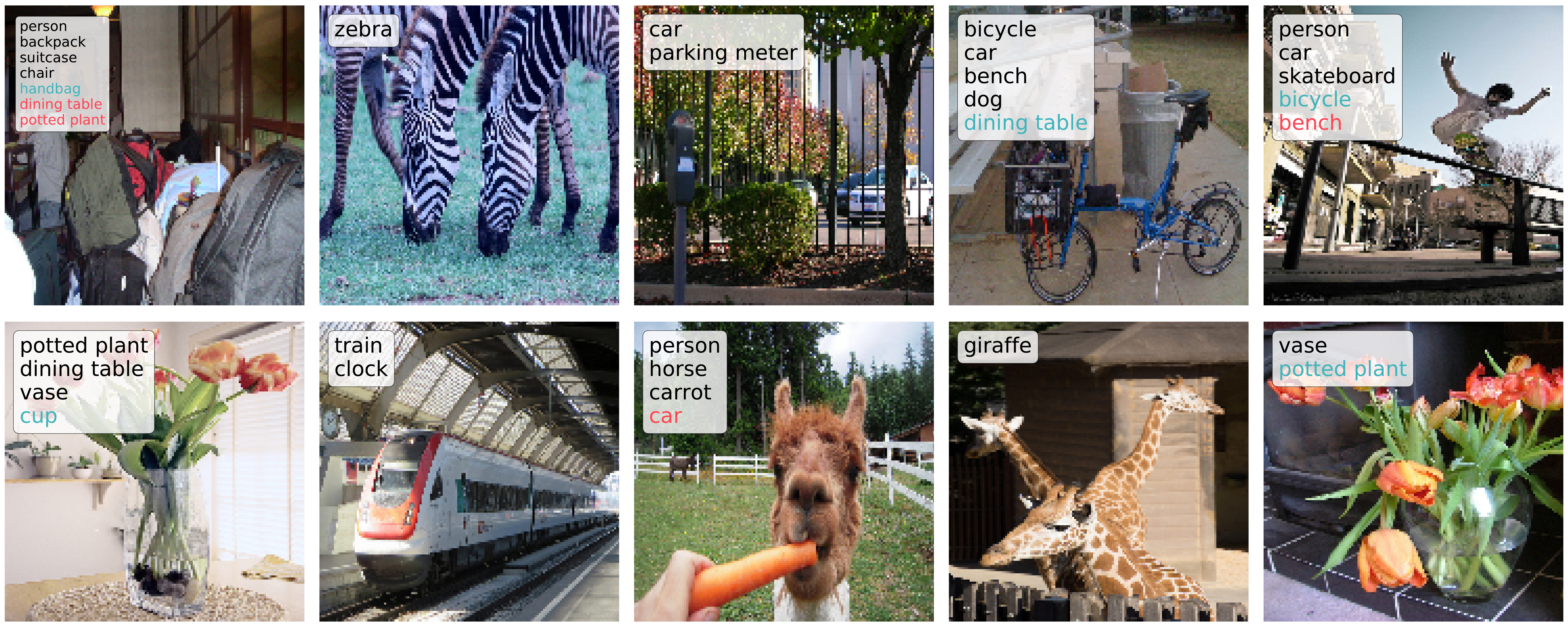}
    \caption{\textbf{Examples of multi-label classification with FDR control} on the MS-COCO dataset. Black classes are true positives, blue classes are spurious, and red classes are missed. The FDR is controlled at level $\alpha=0.1$, $\delta=0.1$.}
    \label{fig:coco-fdr-examples}
\end{figure}

We begin our sequence of examples with a familiar and fundamental setup: multi-label classification.
Here, the features $X_{\rm test}$ can be anything (e.g. an image), and the label $Y_{\rm test} \subseteq \{1,...,K\}$ must be a set of classes (e.g. those contained in the image $X_{\rm test}$).
We have a pre-trained machine learning model $\hat{f}(x)$, which gives us an estimated probability $\hat{f}(x)_k$ that class $k$ is in the corresponding set-valued label.
We will use these probabilities to include the estimated most likely classes in our prediction set,
\begin{equation}
    \T_{\lambda}(x) = \big\{ k : \hat{f}(x)_k > \lambda \big\},\;\; \lambda \in \Lambda
\end{equation}
where $\Lambda=\{0,0.001,...,1\}$ (a discretization of $[0,1]$). 
However, one question remains: \textit{how do we choose $\lambda$}?

LTT will allow us to identify values of $\lambda$ that satisfy a precise probabilistic guarantee---in this case, a bound on the \textit{false-discovery rate} (FDR),
\begin{equation}
    R_{\rm FDR}(\lambda) = \E\left[\underbrace{1 - \frac{\left| Y_{\rm test} \cap \T_{\lambda}(X_{\rm test}) \right|}{\left| \T_{\lambda}(X_{\rm test}) \right|}}_{L_{\rm FDP}(\T_{\lambda}(X_{\rm test}),Y_{\rm test})}\right].
\end{equation}
As annotated in the underbrace, the FDR is the expectation of a loss function, the \emph{false-discovery proportion} (FDP). 
The FDP is low when our prediction set $\T_{\lambda}(X_{\rm test})$ contains mostly elements from $Y_{\rm test}$.
In this sense, the FDR measures the quality of our prediction set: if we have a low FDR, it means most of the elements in our prediction set are good.
By setting $\alpha=0.1$ and $\delta=0.1$, we desire that
\begin{equation}
    \P\left[ R_{\rm FDR}(\lhat) > 0.1 \right] < 0.1,
\end{equation}
where the probability is over the randomness in the calibration set used to pick $\lhat$.

\begin{figure}[h]
    \centering
    \begin{minted}[fontsize=\footnotesize]{python}
# model is a multi-class neural network, X.shape[0]=Y.shape[0]=n
lambdas = torch.linspace(0,1,N) # N can be taken to infinity without penalty
losses = torch.zeros((n,N)) # loss for example i with parameter lambdas[j]
for i in range(n): # In reality we parallelize these loops massively
  sigmoids = model(X[i].unsqueeze(0)).sigmoid().squeeze() # Care with dims
  for j in range(N):
    T = sigmoids > lambdas[j] # This is the prediction set
    set_size = T.float().sum()
    if set_size != 0:
      losses[i,j] = 1 - (T[Y] == True).float().sum()/set_size
risk = losses.mean(dim=0)
pvals = torch.exp(-2*n*(torch.relu(alpha-risk)**2)) # Or the HB p-value
# Fixed-sequence test starting at lambdas[-1] and ending at lambdas[0]
below_delta = (pvals <= delta).float()
valid = torch.tensor([(below_delta[j:].mean() == 1) for j in range(N)])
lambda_hat = lambdas[valid]
    \end{minted}
    \caption{\textbf{PyTorch code for performing FDR control with LTT.}}
    \label{fig:ltt-fdr-code}
\end{figure}

Now that we have set up our problem, we can just run the LTT procedure via the code in Figure~\ref{fig:ltt-fdr-code}.
We use fixed-sequence testing because the FDR is a nearly monotone risk.
In practice, we also wish to use the HB p-value, which is stronger than the simple Hoeffding p-value in Figure~\ref{fig:ltt-fdr-code}.
The result of this procedure on the MS-COCO image dataset is in Figure~\ref{fig:coco-fdr-examples}.

\subsection{Simultaneous Guarantees on OOD Detection and Coverage}
\label{app:ltt-ood}

In our next example, we perform classification with two goals:
\begin{enumerate}
    \item Flag \emph{out-of-distribution} (OOD) inputs without too many false flags.
    \item If an input is deemed \emph{in-distribution} (In-D), output a prediction set that contains the true class with high probability.
\end{enumerate}
Part of the purpose of this example is to teach the reader how to deal with multiple risk functions (one of which is a conditional risk) and a multi-dimensional parameter $\lambda$.

Our setup requires two different models.
The first, ${\rm OOD}(x)$, outputs a scalar that should be larger when the input is OOD.
The second, $\hat{f}(x)_y$, estimates the probability that input $x$ is of class $y$; for example, $\hat{f}(x)$ could represent the softmax outputs of a neural net.
Similarly, the construction of $\Tlam(x)$ has two substeps, each of which uses a different model.
In our first substep, when ${\rm OOD}(x)$ becomes sufficiently large, exceeding $\lambda_1$, we flag the example as OOD by outputting $\emptyset$.
Otherwise, we essentially use the APS method from Section~\ref{subsec:aps} to form prediction sets.
We precisely describe this procedure below:
\begin{equation}
    \T_{\lambda}(x) = \begin{cases}
      \emptyset & \mathrm{OOD}(x) > \lambda_1 \\
      \{\pi_1(x),...,\pi_K(x)\} & \text{else},  \\
    \end{cases} 
\end{equation}
where $K=\inf\{ k : \sum\limits_{j=1}^k\hat{f}(x)_{\pi_j(x)} > \lambda_2\}$ and $\pi(x)$ sorts $\hat{f}(x)$ from greatest to least.
We usually take $\Lambda = \{0,1/N,2/N, ..., 1\}^2$, i.e., we discretize the box $[0,1]\times[0,1]$ into $N^2$ smaller boxes, with $N\approx 1000$.
The intuition of $\T_{\lambda}(x)$ is very simple.
If the example is sufficiently atypical, we give up. 
Otherwise, we create a prediction set using a procedure similar to (but not identical to) conformal prediction.

\begin{figure}[H]
    \centering
    \includegraphics[width=0.6\linewidth]{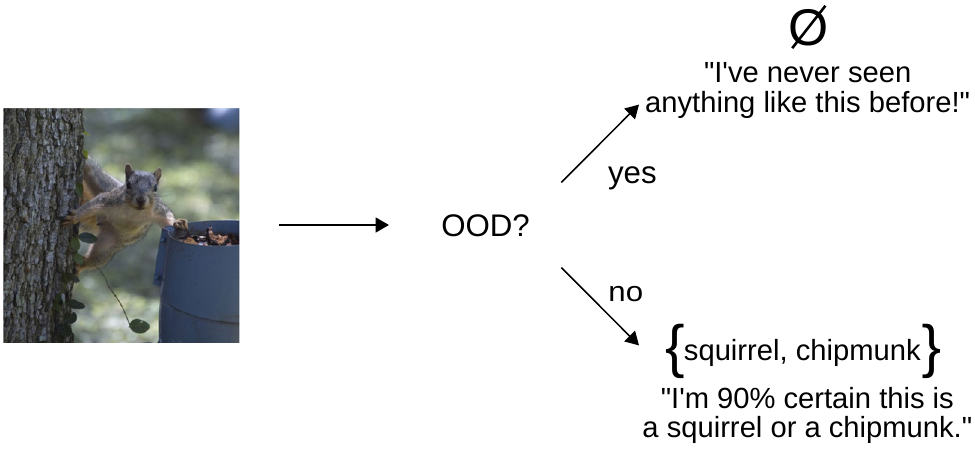}
\end{figure}

Along the same lines, we control two risk functions simultaneously,
\begin{equation}
    R_1(\lambda) = \P\left( \T_{\lambda}(X_{\rm test}) = \emptyset \right) \text{ and } R_2(\lambda) = \P\left( Y_{\rm test} \notin \T_{\lambda}(X_{\rm test}) \; \big\rvert \; \T_{\lambda}(X_{\rm test}) \neq \emptyset \right).
\end{equation}
The first risk function $R_1$ is the probability of a false flag, and the second risk function $R_2$ is the coverage conditionally on being deemed in-distribution.
The user must define risk-tolerances for each, so $\alpha$ is a two-vector, where $\alpha_1$ determines the desired fraction of false flags and $\alpha_2$ determines the desired miscoverage rate.
Setting $\alpha = (0.05,0.1)$ will guarantee that we falsely throw out no more than 5\% of in-distribution data points, and also that among the data points we claim are in-distribution, we will output a prediction set containing the correct class with 90\% probability.
In order to control both risks, we now need to associate a composite null hypothesis to each $\lambda \in \Lambda$.
Namely, we choose
\begin{equation}
    \Hlam : \Hlam^{(1)}\text{ or }\Hlam^{(2)},
\end{equation}
where $\Hlam$ is the union of two intermediate null hypotheses, 
\begin{equation}
    \Hlam^{(1)}: R_1(\lambda) > \alpha_1 \text{ and } \Hlam^{(2)}: R_2(\lambda) > \alpha_2.
\end{equation}
We summarize our setup in the below table.

\begin{table}[H]
    \centering
    \begin{tabular}{c|c|c}
        \textbf{Goal} & \textbf{Null hypothesis} & \textbf{Parameter} \\
        Do not incorrectly label too many images as OOD. & $H_{\lambda}^{(1)}: R_1(\lambda) > \alpha_1$ & $\lambda_1$ \\
        Return a set of labels guaranteed to contain the true one. & $H_{\lambda}^{(2)}: R_2(\lambda) > \alpha_2$ & $\lambda_2$
    \end{tabular}
\end{table}
\begin{figure}[t]
    \centering
    \begin{minted}[fontsize=\footnotesize]{python}
# ood is an OOD detector, model is classifier with softmax output
lambda1s = torch.linspace(0,1,N) # Usually N ~= 1000
lambda2s = torch.linspace(0,1,N)
losses = torch.zeros((2,n,N,N)) # 2 losses, n data points, N x N lambdas
# The following loop can be massively parallelized (and GPU accelerated)
for (i,j,k) in [(i,j,k) for i in range(n) for j in range(N) for k in range(N)]:
  softmaxes = model(X[i].unsqueeze(0)).softmax(1).squeeze() # Care with dims
  cumsum = softmaxes.sort(descending=True)[0].cumsum(0)[Y[i]]
  if odd(X) > lambda1s[j]:
    losses[0,i,j,k] = 1
    continue
  losses[1,i,j,k] = int(cumsum > lambda2s[k])
risks = losses.mean(dim=1) # 2 x N x N
risks[1] = risks[1] - alpha2*risks[0]
pval1s = torch.exp(-2*n*(torch.relu(alpha1-risks[0])**2)) # Or HB p-value
pval2s = torch.exp(-2*n*(torch.relu(alpha2-risks[1])**2)) # Ditto
pvals = torch.maximum(pval1s,pval2s)
# Bonferroni can be replaced by sequential graphical test as in LTT paper
valid = torch.where(pvals <= delta/(N*N))
lambda_hat = [lambda1s[valid[0]], lambda2s[valid[1]]]
    \end{minted}
    \caption{\textbf{PyTorch code for simultaneously controlling the type-1 error of OOD detection and prediction set coverage.}}
    \label{fig:ltt-ood-coverage-code}
\end{figure}
Having completed our setup, we can now apply LTT.
The presence of multiple risks creates some wrinkles, which we will now iron out with the reader.
The null hypothesis $\Hlam$ has a different structure than the ones we saw before, but we can use the same tools to test it.
To start, we produce p-values for the intermediate nulls,
\begin{equation}
    p_{\lambda}^{(1)} = e^{-2n\big(\alpha_1-\Rhat_1(\lambda)\big)_+^2}\text{ and }p_{\lambda}^{(2)} = e^{-2n\big(\alpha_2-\Rhat_2(\lambda)\big)_+^2},
\end{equation}
where
\begin{equation}
    \Rhat_1(\lambda) = \frac{1}{n}\sum\limits_{i=1}^n\ind{\T_\lambda(X_i) = \emptyset} \text{ and } \Rhat_2(\lambda) = \frac{1}{n}\sum\limits_{i=1}^n \ind{Y_i \notin \T_{\lambda}(X_i), \T_{\lambda}(X_i) \neq \emptyset}-\alpha_2\ind{\T_\lambda(X_i) = \emptyset}.\footnote{The second empirical risk, $\Rhat_2$, looks different from a standard empirical risk because of the conditioning. 
    In other words, not all of our calibration data points have nonempty prediction sets; see Section 4 of~\cite{angelopoulos2021learn} to learn more about this point.}
\end{equation}
Since the maximum of two p-values is also a p-value (you can check this manually by verifying its super-uniformity), we can form the p-value for our union null as
\begin{equation}
    p_{\lambda} = \max\Big(p_{\lambda}^{(1)},p_{\lambda}^{(2)}\Big).
\end{equation}
In practice, as before, we use the p-values from the HB inequality as opposed to those from Hoeffding.
Then, instead of Bonferroni correction, we combine them with a less conservative form of sequential graphical testing; see~\cite{angelopoulos2021learn} for these more mathematical details.
For the purposes of this development, it suffices to return the Bonferroni region,
\begin{equation}
    \Lhat = \left\{ \lambda : p_{\lambda} \leq \frac{\delta}{|\Lambda|} \right\}.
\end{equation}
Then, every element of $\Lhat$ controls both risks simultaneously.
See Figure~\ref{fig:ltt-ood-coverage-code} for a PyTorch implementation of this procedure.

\section{Concentration Properties of the Empirical Coverage}
\label{app:empirical-coverage}

We adopt the same notation as Section~\ref{sec:evaluating}.

The variation in $\overline{C}$ has three components. First, $n$ is finite. We analyzed how this leads to fluctuations in the coverage in Section~\ref{subsec:cal_size}.
The second source of fluctuations is the finiteness of $n_{\textnormal{val}}$, the size of the validation set.
A small number of validation points can result in a high-variance estimate of the coverage.
This makes the histogram of the $C_j$ wider than the beta distribution above. 
However, as we will now show, $C_j$ has an analytical distribution that allows us to exactly understand the histogram's expected properties.

We now examine the distribution of $C_j$.
Because $C_j$ is an average of indicator functions, it looks like it is a binomially distributed random variable.
This is true conditionally on the calibration data, but not marginally.
This is because the mean of the binomial is beta distributed; as we showed in the above analysis, $\E\left[C_j \big\rvert \{(X_{i,j},Y_{i,j})\}_{i=1}^n \right] \sim \mathrm{Beta}(n+1-l,l)$, where $(X_{i,j},Y_{i,j})$ is the $i$th calibration point in the $j$th trial.
Conveniently, binomial random variables with beta-distributed mean,
\begin{equation}
    C_j \sim \frac{1}{n_{\textnormal{val}}}\mathrm{Binom}(n_{\textnormal{val}},\mu) \text{ where } \mu \sim \mathrm{Beta}(n+1-l,l),
\end{equation}
are called \emph{beta-binomial} random variables.
We refer to this distribution as $\mathrm{BetaBinom}(n_{\textnormal{val}},n+1-l,l)$; its properties, such as moments and probability mass function, can be found in standard references.

Knowing the analytic form of the $C_j$ allows us to directly plot its distribution. 
After a sufficient number of trials $R$, the histogram of $C_j$ should converge almost exactly to its analytical PMF (which is only a function of $\alpha$, $n$, and $n_{\textnormal{val}}$).
The plot in Figure~\ref{fig:betabinom} shows how the histograms should look with different values of $n_{\textnormal{val}}$ and large $R$.
Code for producing these plots is also available in the aforementioned Jupyter notebook.

\begin{figure}[H]
    \centering
    \includegraphics[width=0.6\linewidth]{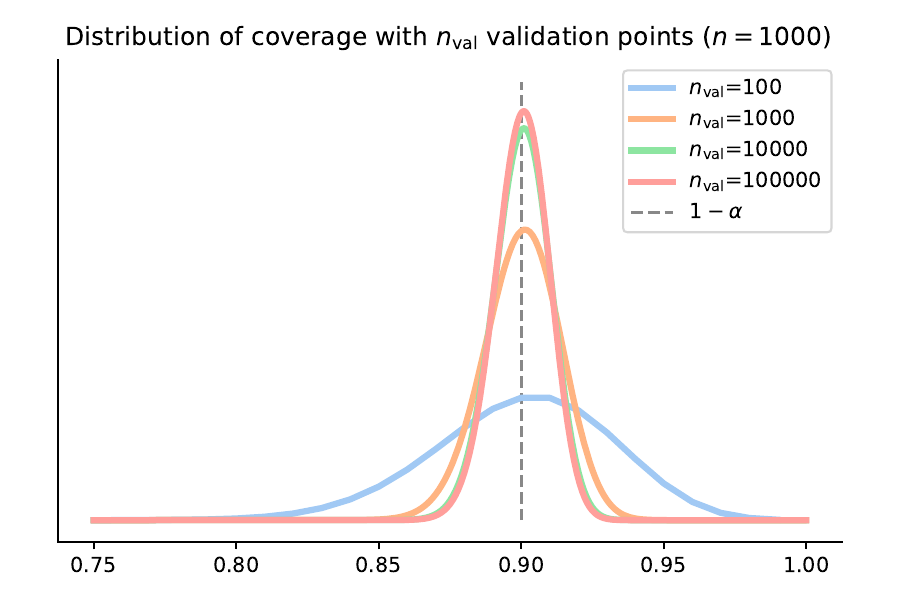}
    \caption{\textbf{The distribution of empirical coverage} converges to the Beta distribution in Figure~\ref{fig:beta} as $n_{\textnormal{val}}$ grows. However, for small values of $n_{\textnormal{val}}$, the histogram can have an inflated variance.}
    \label{fig:betabinom}
\end{figure}

The final source of fluctuations is due to the finite number of experiments, $R$.
We have now shown that the $C_j$ are independent beta-binomial random variables.
Unfortunately, the distribution of $\overline{C}$---the mean of $R$ independent beta-binomial random variables---does not have a closed form. 
However, we can simulate the distribution easily, and we visualize it for several realistic choices of $R$, $n_{\textnormal{val}}$, and $n$ in Figure~\ref{fig:average-empirical-coverage}.
\begin{figure}[H]
    \centering
    \includegraphics[width=0.49\linewidth]{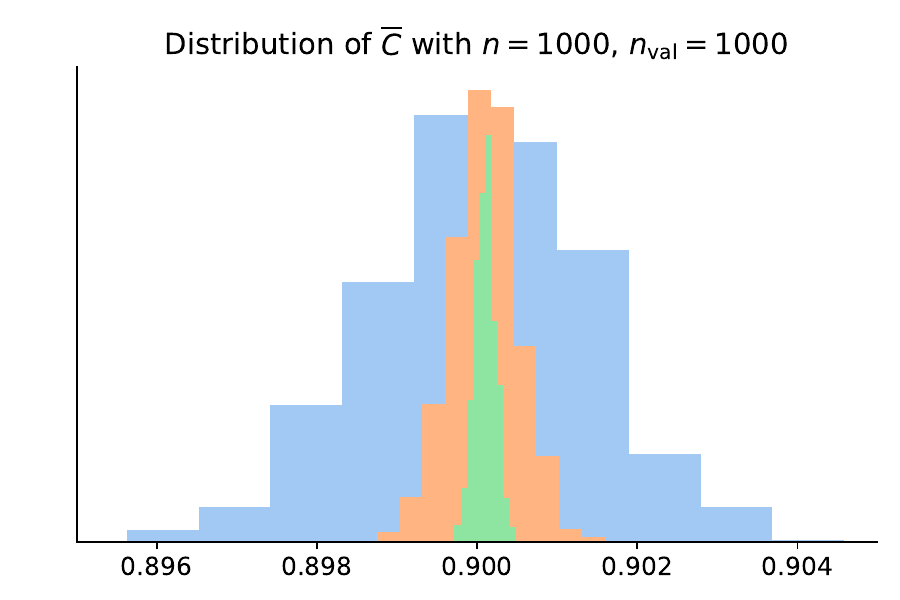}
    \includegraphics[width=0.49\linewidth]{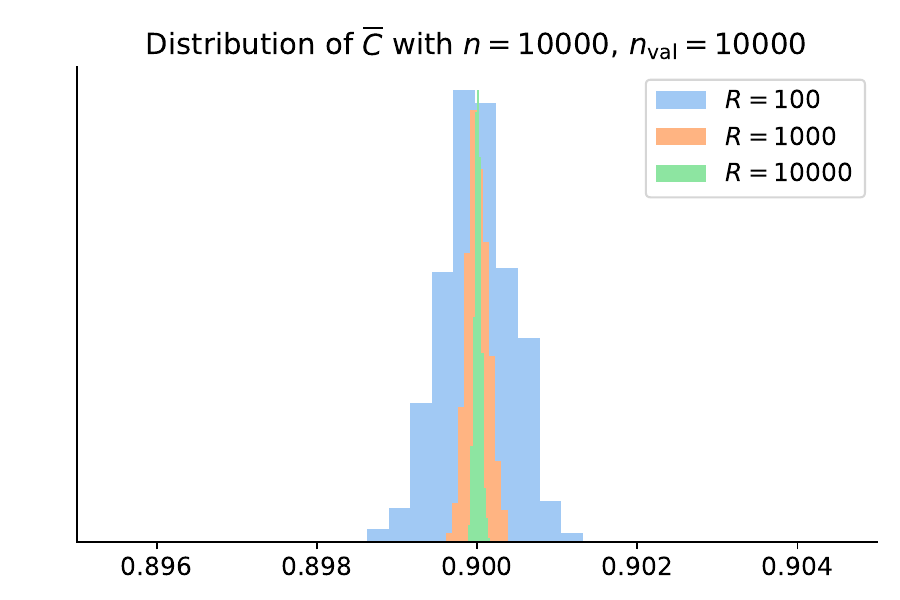}
    \caption{\textbf{The distribution of average empirical coverage} over $R$ trials with $n$ calibration points and $n_{\textnormal{val}}$ validation points. \jupyter{https://github.com/aangelopoulos/conformal-prediction/blob/main/notebooks/correctness_checks.ipynb}}
    \label{fig:average-empirical-coverage}
\end{figure}
Furthermore, we can analytically reason about the tail properties of $\overline{C}$.
Since $\overline{C}$ is the average of $R$ i.i.d. beta-binomial random variables, its mean and standard deviation are
\begin{equation}
    \E\Big(\overline{C}\Big) = 1-\frac{l}{n+1} \; \; \; \text{   and   } \; \; \;
    \sqrt{\mathrm{Var}\Big( \overline{C}} \Big) = \sqrt{\frac{l(n+1-l)(n+n_{\textnormal{val}}+1)}{n_{\textnormal{val}}R(n+1)^2(n+2)}} = \mathcal{O}\left( \frac{1}{\sqrt{R\min(n,n_{\textnormal{val}})}} \right).
\end{equation}

The best way for a practitioner to carefully debug their procedure is to compute $\overline{C}$ empirically, and then cross-reference with Figure~\ref{fig:average-empirical-coverage}. We give code to simulate histograms with any $n$, $R$, and $n_{\textnormal{val}}$ in the linked notebook of Figure~\ref{fig:average-empirical-coverage}.
If the simulated average empirical coverage does not align well with the coverage observed on the real data, there is likely a problem in the conformal implementation.

\section{Theorem and Proof: Coverage Property of Conformal Prediction}
\label{app:coverage-proof}
This is a standard proof of validity for split-conformal prediction first appearing in~\cite{papadopoulos2002inductive}, but we reproduce it here for completeness.
Let us begin with the lower bound.

\begin{theorem}[Conformal calibration coverage guarantee]
Suppose $(X_i, Y_i)_{i = 1,\dots,n}$ and $(X_{\rm test}, Y_{\rm test})$ are i.i.d.
Then define $\hat{q}$ as 
\begin{equation}
    \hat{q} = \inf \left\{q : \frac{|\{i : s(X_i, Y_i) \le q\}|}{n} \geq \frac{\lceil (n+1)(1-\alpha) \rceil}{n} \right\}.
\end{equation}
and the resulting prediction sets as
\begin{equation*}
    \C(X) = \left\{y : s(X,y) \le \hat{q}\right\}.
\end{equation*}
Then,
\begin{equation*}
P\Big(Y_{\rm test} \in \C(X_{\rm test})\Big) \ge 1 - \alpha.
\end{equation*}

\end{theorem}

This is the same coverage property as \eqref{eq:coverage} in the introduction, but written more formally.
As a technical remark, the theorem also holds if the observations to satisfy the weaker condition of exchangeability; see \cite{vovk2005algorithmic}. 
Below, we prove the lower bound.
\begin{proof}[Proof of Theorem~\ref{thm:conformal_calibration}]
Let $s_i = s(X_i, Y_i)$ for $i = 1,\dots,n$ and $s_{\rm test} = s(X_{\rm test}, Y_{\rm test})$. To avoid handling ties, we consider the case where the $s_i$ are distinct with probability $1$. See~\cite{tibshirani2019conformal} for a proof in the general case.

Without loss of generality we assume the calibration scores are sorted so that $s_1 < \dots < s_n$. In this case, we have that $\hat{q} = s_{\lceil (n+1)(1-\alpha) \rceil}$ when $\alpha \geq \frac{1}{n+1}$ and $\hat{q}=\infty$ otherwise. 
Note that in the case $\hat{q}=\infty$, $\C(X_{\rm test}) = \Y$, so the coverage property is trivially satisfied; thus, we only have to handle the case when $\alpha \geq \frac{1}{n+1}$.
We proceed by noticing the equality of the two events
\begin{equation}
\{Y_{\rm test} \in \C(X_{\rm test})\} = \{s_{\rm test} \le \hat{q}\}.
\end{equation}
Combining this with the definition of $\hat{q}$ yields
\begin{equation}
\{Y_{\rm test} \in \C(X_{\rm test})\} = \{s_{\rm test} \leq s_{\lceil (n+1)(1-\alpha)\rceil}\}.
\end{equation}
Now comes the crucial insight. By exchangeability of the variables $(X_1,Y_1),\dots,(X_{\rm test},Y_{\rm test})$, we have
\begin{equation*}
    P(s_{\rm test} \le s_k) = \frac{k}{n+1}
\end{equation*}
for any integer $k$. In words, $s_{\rm test}$ is equally likely to fall in anywhere between the calibration points $s_1,\dots,s_n$. Note that above, the randomness is over all variables $s_1,\dots,s_{n},s_{\rm test}$

From here, we conclude
\begin{equation*}
    P( s_{\rm test} \le s_{\lceil (n+1)(1-\alpha))\rceil}) = \frac{\lceil (n+1)(1-\alpha)\rceil}{(n+1)} \ge 1-\alpha,
\end{equation*}
which implies the desired result.
\end{proof}

Now we will discuss the upper bound.
Technically, the upper bound only holds when the distribution of the conformal score is continuous, avoiding ties. 
In practice, however, this condition is not important, because the user can always add a vanishing amount of random noise to the score.
We will state the theorem now, and defer its proof.

\begin{theorem}[Conformal calibration upper bound]
\label{thm:upper-bound}
    Additionally, if the scores $s_1, ..., s_n$ have a continuous joint distribution, then
    \begin{equation*}
        P\Big(Y_{\rm test} \in \C(X_{\rm test}, U_{\rm test}, \hat{q})\Big) \leq 1 - \alpha + \frac{1}{n+1}.
    \end{equation*}
\end{theorem}

\begin{proof}
    See Theorem 2.2 of 
    \cite{lei2018distribution}.
\end{proof}

\end{document}